\documentclass{article}

\usepackage[final]{neurips_2025}

\usepackage[utf8]{inputenc} 
\usepackage[T1]{fontenc}    
\usepackage{hyperref}[citecolor=magenta]
\hypersetup{
    colorlinks = true,
    citecolor = {magenta},
}

\usepackage{fontawesome5} 

\usepackage{url}            
\usepackage{booktabs}       
\usepackage{amsfonts}       
\usepackage{nicefrac}       
\usepackage{microtype}      
\usepackage{xcolor}         

\usepackage{graphicx}
\usepackage{subfigure}
\usepackage{enumitem}  
\usepackage{amsmath}
\usepackage{amssymb}
\usepackage{mathtools}
\usepackage{amsthm}
\usepackage{multirow}
\usepackage[dvipsnames]{xcolor}
\usepackage{colortbl} 
\definecolor{lightgray}{gray}{0.9}
\definecolor{deepgreen}{RGB}{50,180,50}  
\definecolor{deepred}{RGB}{220,50,50}    

\usepackage{wrapfig}

\usepackage{algorithm}
\usepackage{algpseudocode}

\theoremstyle{plain}
\newtheorem{theorem}{Theorem}[section]

\newtheorem{lemma}[theorem]{Lemma}

\theoremstyle{definition}
\newtheorem{definition}[theorem]{Definition}

\theoremstyle{remark}

\title{Fuz-RL: A Fuzzy-Guided Robust Framework for Safe Reinforcement Learning under Uncertainty}

\author{%
  Xu Wan \\
  Zhejiang University\\
  \texttt{wanxu@zju.edu.cn} \\
  \And
  Chao Yang \\
  Alibaba DAMO Academy \\
  \texttt{xiuxin.yc@alibaba-inc.com} \\
  \AND
  Cheng Yang \\
  Alibaba DAMO Academy \\
  \texttt{charis.yangc@alibaba-inc.com} \\
  \And
  Jie Song \\
  Peking University \\
  \texttt{jie.song@pku.edu.cn} \\
  \And
  Mingyang Sun\thanks{Corresponding author} \\
  Peking University \\
  \texttt{smy@pku.edu.cn} \\
}

\begin{document}

\maketitle

\begin{abstract}
Safe Reinforcement Learning (RL) is crucial for achieving high performance while ensuring safety in real-world applications. However, the complex interplay of multiple uncertainty sources in real environments poses significant challenges for interpretable risk assessment and robust decision-making. To address these challenges, we propose \textbf{Fuz-RL}, a fuzzy measure-guided robust framework for safe RL. 
Specifically, our framework develops a novel fuzzy Bellman operator for estimating robust value functions using Choquet integrals. Theoretically, we prove that solving the Fuz-RL problem (in Constrained Markov Decision Process (CMDP) form) is equivalent to solving distributionally robust safe RL problems (in robust CMDP form), effectively avoiding min-max optimization. Empirical analyses on safe-control-gym and safety-gymnasium scenarios demonstrate that Fuz-RL effectively integrates with existing safe RL baselines in a model-free manner, significantly improving both safety and control performance under various types of uncertainties in observation, action, and dynamics. 

\faGithub \ \textit{\textbf{The code is available in \href{https://github.com/waunx/FuzRL}{Here}.}}

\end{abstract}

\section{Introduction}

While safe reinforcement learning (RL) has achieved remarkable success in safety-crucial decision-making tasks, deploying safe RL in real-world applications remains challenging due to multiple sources of uncertainty \cite{as2022constrained, chow2018risk, wachi2024safe}. Recent methods using Lyapunov functions and reachability analysis provide theoretical safety guarantees for control tasks \cite{ames2016control, chow2018lyapunov, wang2023enforcing, yu2022reachability,wan2023adapsafe,wan2024adapsafe2}, but focus primarily on idealized, deterministic settings. These approaches struggle with the complex, coupled uncertainties of real-world systems, including sensor noise, actuator delays, and environmental variations.

Existing robust approaches to safe RL face several key limitations for real-world tasks. Traditional min-max techniques \cite{yu1996max, morimoto2005robust, tanabe2022max} focus on worst-case scenarios, resulting in overly conservative policies and computational intractability. While distributionally robust methods attempt to model uncertainty distributions, they typically assume simplified, independent uncertainties through KL-divergence constraints \cite{smirnova2019distributionally} or Gaussian perturbations \cite{tessler2019action}, and treat different perturbations with equal importance. Risk-sensitive approaches using probability measures like conditional Value-at-Risk (VaR) \cite{singh2020improving}, Wang transform \cite{queeney2024risk}, and Entropic VaR \cite{ni2022risk} enable uncertainty quantification through coherent risk functionals, require careful parameter tuning and struggle to handle multiple noise sources effectively.

However, when multiple uncertainties are correlated and converge on a single system component, the resultant performance degradation often exhibits super-additive behavior. 
To handle such coupled uncertainties, \textit{fuzzy measure theory} has shown promise in various decision-making tasks, from robust motion planning \cite{cho2002risk, zhao2020multi, hentout2023review} to adaptive control \cite{liu2015fuzzy, guven2024exploring}, through its ability to model non-additive effects and provide clear behavioral interpretations. While these successes demonstrate the potential of fuzzy measures for uncertainty quantification, extending this approach to constrained Markov decision process (CMDP) remains challenging, particularly in balancing performance objectives with safety constraints under uncertainty.

Motivated by this, we propose a novel \textbf{Fuz}zy-guided framework for Safe \textbf{RL} (Fuz-RL) that unifies uncertainty quantification and enhances current safe RL's robustness through fuzzy theory. Specifically, our main contributions are:

(1) We introduce a novel fuzzy Bellman operator that integrates Choquet integrals of fuzzy measure into value function to achieve robust value estimation under potential perturbations.

(2) We provide robustness equivalence for our Fuz-RL framework by demonstrating that solving Fuz-RL problem (a CMDP form) is equivalent to distributionally robust safe RL problems (a robust CMDP form).

(3) By seamlessly integrating the Fuz-RL framework into three safe RL methods, we conduct several robust assessments involving observation, action, and dynamics uncertainty for safe-control-gym tasks and safety-gymnasium tasks. As expected, Fuz-RL significantly enhances the robustness of safe RL algorithms in multi-source uncertainty scenarios.

\section{Related Work}
Our work builds upon and connects two main research directions: robust approaches in safe RL and fuzzy-based uncertainty quantification in MDPs. We review relevant work in these areas and highlight the research gaps our work addresses.

\paragraph{Robust Approaches in Safe RL.}
Uncertainties within the CMDP framework manifest in various forms, including state shifts \cite{korkmaz2022deep}, action disturbances \cite{tessler2019action}, and dynamics uncertainty \cite{pinto2017robust}. To address these challenges, distributionally robust optimization has been employed, where policies are optimized against worst-case transition kernels within a Wasserstein ambiguity set \cite{smirnova2019distributionally, liu2022distributionally}.
An alternative direction leverages risk-sensitive measures to enhance robustness under safety constraints. For instance, Conditional Value-at-Risk (CVaR) has been integrated into policy optimization to explicitly balance expected return and worst-case performance \cite{ying2021towards}, while coherent distortion risk measures offer formal robustness guarantees in safe RL \cite{queeney2024risk}.
Other approaches focus on adversarial robustness or model-based safety. Some works combine robust model predictive control (MPC) with tube-based constraints to ensure recursive feasibility under uncertainty \cite{zanon2020safe}. Gaussian Processes have also been used as safety oracles in model-based RL to probabilistically identify constraints \cite{airaldi2023learning}. Adversarially robust methods further address observation perturbations via state-adversarial MDPs and policy regularization \cite{li2024safe, liang2022efficient, zhang2020robust}.

\paragraph{Fuzzy Measures in MDPs.} 
Fuzzy logic provides an interpretable framework for quantifying and managing uncertainty in complex systems. Zadeh's fuzzy sets theory \cite{zadeh1965fuzzy} laid the foundation for uncertainty measures. Building on this, possibility theory \cite{dubois1988possibility} emerged as a significant fuzzy approach to uncertainty quantification, offering an alternative to probabilistic methods. Then, \cite{liu2002fuzzy} introduced credibility theory, which combines fuzzy and probability measures to create a self-dual measure. For RL community, fuzzy Q-learning \cite{glorennec1997fuzzy} and possibilistic MDPs \cite{sabbadin2001possibilistic} incorporate fuzzy logic into MDPs. Furthermore, \cite{hein2017particle}, \cite{huang2020interpretable} and \cite{hostetter2023self} developed various fuzzy RL approaches that provide enhanced interpretability and effectiveness compared to deep neural network-based RL methods.
\cite{guo2018fuzzy} introduces $m_{\lambda}$ measure, which combined the possibility measure and necessity measure to balance optimism and pessimism in decision-making systems, which has shown promise in chance-constrained programming. Recent advancements have further expanded the application of fuzzy logic in uncertainty modeling. For instance, \cite{hua2023fuzzy} developed a fuzzy adaptive sliding mode control method for robotic systems with uncertainties, \cite{seoni2023application} conducted a comprehensive review of uncertainty quantification applications for healthcare. 

However, incorporating fuzzy logic for robustness enhancement in safe RL remains unexplored. Inspired by the fuzzy-guided $m_{\lambda}$ fuzzy measure\cite{guo2018fuzzy}, we aim to achieve a robust risk-aversion and reward-pursuitin value estimation through fuzzy logic for safe RL.

\section{Preliminary}

\subsection{Robust CMDP}
We consider formulating the safe RL problem as an infinite-horizon CMDP \cite{altman1999constrained}, which is defined by the tuple $(\mathcal{S}, \mathcal{A}, p, r, c, \gamma, d_0)$, where $\mathcal{S}$ is the finite state space and $\mathcal{A}$ is the action space. $p: \mathcal{S} \times \mathcal{A} \times \mathcal{S} \to [0, 1]$ is the transition model, $r, c: \mathcal{S} \times \mathcal{A} \times \mathcal{S} \to \mathbb{R}$ are the bounded reward function and cost function defining the objective and constraint, respectively. $\gamma \in [0, 1)$ is the discount factor, and $d_0: \mathcal{S} \to [0, 1]$ is the initial state distribution. A policy $\pi: \mathcal{S} \to \Delta(\mathcal{A})$ maps states to distributions over actions.

To address system uncertainty, we introduce a robust formulation of CMDP. Following the concept of $(s,a)$-rectangular uncertainty sets \cite{nilim2005robust}, we define the uncertainty set $\mathcal{P}$ as:
\begin{equation}
\mathcal{P} = \prod_{s,a} \mathcal{P}_s^a, \quad \mathcal{P}_s^a \subseteq \Delta(\mathcal{S})
\end{equation}
where $\mathcal{P}_{s}^{a}$ represents the set of all possible transition probabilities over $\mathcal{S}$ for a given state-action pair $(s,a)$. For $\forall s \in \mathcal{S}, a \in \mathcal{A}$, we define:
\begin{equation}
\mathcal{P}_s^a = \{p(\cdot|s,a) : d(p(\cdot|s,a), p_0(\cdot|s,a)) \leq \epsilon\}
\end{equation}
where $p_0$ is the nominal transition model, $d(\cdot, \cdot)$ is a distance metric, and $\epsilon > 0$ defines the uncertainty radius.

The objective of the robust CMDP is to find a policy $\pi$ that solves the following constrained optimization problem:
\begin{equation}
    \label{eq:robust_cmdp}
    \begin{aligned}
        \max_{\pi} \min_{p \in \mathcal{P}}  \mathbb{E}_{\tau \sim (\pi, p)} \left[ \sum_{t=0}^{\infty} \gamma^t r(s_t, a_t) \right] 
    \quad \text{s.t.} \
    \max_{p \in \mathcal{P}} \ \mathbb{E}_{\tau \sim (\pi,p)} \left[ \sum_{t=0}^{\infty} \gamma^t c(s_t, a_t) \right] \leq B 
    \end{aligned}
\end{equation}
where $\tau \sim (\pi, p)$ denotes trajectories sampled according to $s_0 \sim d_0$, $a_t \sim \pi(\cdot | s_t)$ and $s_{t+1} \sim p(\cdot | s_t, a_t)$, and $B > 0$ is the safety budget constraint.

For computational tractability, we partition the uncertainty set $\mathcal{P}_s^a$ into $K$ distinct levels:
\begin{equation}
\begin{aligned}
\mathcal{P}_s^a = \bigcup_{k=1}^K \mathcal{P}_{s,k}^a, \quad
\mathcal{P}_{s,k}^a = \{p(\cdot|s,a) : \epsilon_{k-1} < d(p(\cdot|s,a), p_0(\cdot|s,a)) \leq \epsilon_k \}
\end{aligned}
\label{eq:uncertainty_set}
\end{equation}
where $0 = \epsilon_0 < \epsilon_1 < ... < \epsilon_K = \epsilon$ defines a sequence of increasing uncertainty thresholds. 

\subsection{Fuzzy Measures Fundamentals}

Traditional probability measures treat uncertainties in a purely additive manner, assuming independent effects from different uncertainties. However, in real control systems, as the distance from nominal dynamics increases, the compound effects of uncertainties often exhibit super-additive behavior. For example, when considering two uncertainty sources $A$ and $B$, their joint impact on system performance may be greater than the sum of their individual effects:
\begin{equation}
m(A \cup B) > m(A) + m(B)
\end{equation}
Moreover, as the system deviates further from the nominal model, the impact on performance and safety constraints typically grows non-linearly. 

To capture such non-additive effects, we first introduce the concept of fuzzy measure, which provides an interpretable way to assess the impacts of uncertainty by assigning non-additive weights to combinations of uncertainty levels. The formal definition is as follows:

\begin{definition}[\textit{Fuzzy Measure}~\cite{murofushi2000fuzzy}]
\label{def:fuzzy_measure}
For each state-action pair $(s,a)$, a fuzzy measure $m$ is a function $m: 2^{\{\mathcal{P}_{s,1}^a, \mathcal{P}_{s,2}^a, \ldots, \mathcal{P}_{s,K}^a\}} \rightarrow [0,1]$, satisfying:

    (1) $m(\emptyset) = 0, \quad m({\{\mathcal{P}_{s,1}^a, \mathcal{P}_{s,2}^a, \ldots, \mathcal{P}_{s,K}^a\}}) = 1$,
    
    (2) If $A \subseteq B \subseteq \{1, 2, ..., K\}$, then $m(\mathcal{P}_{s,A}^a) \leq m(\mathcal{P}_{s,B}^a)$ (monotonicity).
\end{definition}

Measuring uncertainty impacts for all possible subset combinations in $2^{\{\mathcal{P}_{s,1}^a, \mathcal{P}_{s,2}^a, \ldots, \mathcal{P}_{s,K}^a\}}$ is computationally intractable, as it requires an exponential number of samples. To address this computational challenge while preserving the ability to model super-additive effects, we adopt the $\lambda$-fuzzy measure, which offers an efficient parameterization of subset relationships:

\begin{definition}[$\lambda$\textit{-Fuzzy Measure }\cite{denneberg1994non}]
\label{def:sugeno}
A $\lambda$-fuzzy measure satisfies, for all disjoint subsets $\mathcal{P}_{s,A}^a,\mathcal{P}_{s,B}^a$:
\begin{equation}
\label{eq:sugeno}
\begin{aligned}
    m(\mathcal{P}_{s,A \cup B}^a) = m(\mathcal{P}_{s,A}^a) + m(\mathcal{P}_{s,B}^a)
    + \lambda \, m(\mathcal{P}_{s,A}^a) \, m(\mathcal{P}_{s,B}^a),
\end{aligned}
\end{equation}
where $\lambda > -1$ determines the degree of interaction.  When $\lambda \in (-1, 0)$, the measure exhibits \textit{sub-additive} behavior; when $\lambda \in (0, \infty)$, it models \textit{super-additive} effects among different uncertainties. Obviously, if $\lambda = 0$, then a $\lambda$-fuzzy measure is a normalized additive measure, i.e., a probability measure.
\end{definition}

The connection between $\lambda$-fuzzy measures and robust optimization is established through the Choquet integral's pessimistic characterization:

\begin{lemma}[\textit{Choquet Integral Representation} \cite{gilboa1994additive}]
\label{lemma:choquet_pessimism}
For any bounded measurable function $f: \Omega \rightarrow \mathbb{R}$ and $\lambda$-fuzzy measure $m$ with $\lambda \geq 0$:
\[
(C)\int_{\Omega} f\, dm = \min_{P \in \text{core}(m)} \mathbb{E}_P [f],
\]
where $\text{core}(m) = \{ P \in \mathcal{P}(\Omega): P(A) \geq m(A) \}$ is the set of probability measures dominating $m$.
\end{lemma}

\section{Fuzzy Measure-based Robust Safe RL Framework}

\subsection{Theoretical Foundation of Fuz-RL}
In this section, we connect fuzzy measure with robust CMDP by introducing the \emph{Fuzzy Bellman operator}.

\paragraph{Fuzzy Bellman Operator.}
Leveraging Lemma \ref{lemma:choquet_pessimism}, we define the fuzzy Bellman operator that encodes worst-case scenarios through the Choquet integral in Definition \ref{def:fuzzy_bellman}:

\begin{definition}[\textit{Fuzzy Bellman Operator}]
\label{def:fuzzy_bellman}
Let $\mathcal{B}(\mathcal{S})$ denote the space of bounded measurable value functions $V$ on $\mathcal{S}$. The fuzzy Bellman operator $\mathcal{F}: \mathcal{B}(\mathcal{S}) \rightarrow \mathcal{B}(\mathcal{S})$ is defined as:
\[
\mathcal{F}(V)(s) = \mathbb{E}_{a\sim \pi} \Big[r(s,a) + \gamma \ (C)\int_{\mathcal{P}_s^a} \mathbb{E}_{s' \sim p}[V(s')]\, dm(p)\Big].
\]
\end{definition}
where $m(\cdot)$ is the convex fuzzy meassure.

Furthermore, we demonstrate that the fuzzy Bellman operator maintains fundamental properties of the standard Bellman operator ($\gamma$-contraction and convergence) when integrated with value functions as detailed in Theorem \ref{thm:gamma_contraction} and Theorem \ref{thm:convergence}. Therefore, the fuzzy Bellman operator can be seamlessly integrated into value functions, enabling the establishment of robust value estimation with theoretical guarantees.

\begin{theorem}[\textit{$\gamma$-contraction of Fuzzy Bellman Operator}]
\label{thm:gamma_contraction}
For any $V_1, V_2 \in \mathcal{B}(\mathcal{S})$,
\[
\|\mathcal{F}(V_1) - \mathcal{F}(V_2)\|_{\infty} \leq \gamma \|V_1 - V_2\|_{\infty}.
\]
\end{theorem}

\begin{theorem}[\textit{Convergence of Fuzzy Bellman Operator}]
\label{thm:convergence}
Let $V^0 \in \mathcal{B}(\mathcal{S})$ be an initial value function and $V^{n+1} = \mathcal{F}(V^n)$. Then $V^n$ converges to a unique fixed point $V^*$ satisfying $V^*=\mathcal{F}(V^*)$ with geometric rate $\|V^n - V^*\|_\infty \leq \gamma^n \|V^0 - V^*\|_\infty$.
\end{theorem}

\paragraph{Robust Equivalence.}
Applying Lemma \ref{lemma:choquet_pessimism} shows that the Choquet integral automatically encodes a robust perspective via the fuzzy measure $m$. Consequently, we can prove the following Theorem \ref{theorem:equivalent}:

\begin{theorem}[\textit{Equivalent Theorem}]
\label{theorem:equivalent}

Let $m$ be a convex $\lambda$-fuzzy measure on $\mathcal{P}_s^a$
 defined by Definition \ref{def:sugeno} such that:
(1) $core(m) \subseteq \mathcal{P}$, (2) $\arg\min_{p \in \mathcal{P}} \mathbb{E}[r] \in core(m)$, (3) $\arg\max_{p \in \mathcal{P}} \mathbb{E}[c] \in core(m)$. Let $m'(A) := 1 - m(\mathcal{P}^a_s \setminus A)$ is the dual fuzzy measure of $m$.

Then the \textbf{fuzzy robust safe RL problem (CMDP form)}:
\begin{align}
    \max_{\pi} J^\mathcal{F}_r(\pi) \quad \text{s.t.} \quad J^\mathcal{F}_c(\pi) \leq B
\end{align}
where
\begin{align}
    J^\mathcal{F}_r(\pi) &= \mathbb{E}_{s_0 \sim d_0}\left[\ (C)\int_{\mathcal{P}^a_s} \sum_{t=0}^{\infty} \gamma^t r(s_t, \pi(s_t)) \, dm(p) \right] \\
    J^\mathcal{F}_c(\pi) &= \mathbb{E}_{s_0 \sim d_0}\left[\ (C)\int_{\mathcal{P}^a_s} \sum_{t=0}^{\infty} \gamma^t c(s_t, \pi(s_t)) \, dm'(p) \right]
\end{align}
\textbf{is equivalent to the distributionally robust safe RL problem (robust CMDP form)}:
\begin{align}
    \max_{\pi} \min_{p \in \mathcal{P}} \mathbb{E}_{({\pi, p})}[r] \quad \text{s.t.} \quad \max_{p \in \mathcal{P}} \mathbb{E}_{({\pi, p})}[c] \leq B
\end{align}
\end{theorem}

The detailed proofs of Theorem \ref{thm:gamma_contraction}, Theorem \ref{thm:convergence} and Theorem \ref{theorem:equivalent} are provided in Appendix \ref{app:proofs}.

\subsection{Practical Implementation of Fuz-RL}
\label{sec:implementation}

Having established the fuzzy Bellman operator and its theoretical properties, we now describe how to implement Fuz-RL in practice. The pseudo-code of Fuz-RL is detailed in Appendix~\ref{app::alg} Algorithm~\ref{alg: Fuz_RL}.

\paragraph{Estimation of Fuzzy Measures.}
To operationalize the fuzzy Bellman operator, we require an efficient method for estimating fuzzy measures $m(\cdot)$ that capture uncertainty impacts across different perturbation subsets. We adopt a neural network-based approach that learns fuzzy measure densities directly from state representations.

For each state $s$, we employ a fuzzy network $\phi$ parameterized as a multi-layer perceptron (MLP) with two hidden layers. The network takes the state vector as input and outputs fuzzy density parameters $\{g_k\}_{k=1}^K$:
\begin{equation}
\mathbf{g} = \sigma(\text{MLP}_\phi(s)),
\label{eq:mlp_fuzzy}
\end{equation}
where $\sigma(\cdot)$ denotes the softmax activation function applied to the network output. To maintain numerical stability, we apply clamping to constrain the fuzzy values within $[\epsilon, 1-\epsilon]$ where $\epsilon = 10^{-4}$, preventing degenerate solutions while satisfying the mathematical constraints of $\lambda$-fuzzy measures.

Given the learned densities $\mathbf{g} = (g_1, \ldots, g_K)$, the interaction parameter $\lambda$ is determined by solving the characteristic equation:
\begin{equation}
\prod_{k=1}^K(1 + \lambda g_k) = 1 + \lambda,
\label{eq:lambda_equation}
\end{equation}
using a hybrid bisection-Newton method with gradient detachment to ensure numerical stability. Once $\lambda$ is obtained, the fuzzy measure for any subset $A \subseteq \{1,\ldots,K\}$ can be computed via the $\lambda$-rule:
\begin{equation}
m(A) = \frac{\prod_{k \in A} (1 + \lambda g_k) - 1}{\lambda}.
\label{eq:lambda_measure}
\end{equation}
To simulate the impact of uncertain system state transitions across different uncertainty levels, we employ a stratified perturbation sampling strategy. For each uncertainty level $k \in \{1, \ldots, K\}$, we apply independent isotropic Gaussian perturbations:
\begin{equation}
\tilde{s}_{k} = s + \epsilon_k \cdot \mathbf{n}_k, \quad \mathbf{n}_k \sim \mathcal{N}(0, I),
\label{eq:gaussian_perturbation}
\end{equation}
where $\epsilon_k$ represents the perturbation scale for uncertainty level $k$, typically set as $\epsilon_k = \epsilon_{\text{base}} \cdot k$ to create a hierarchy of perturbation intensities. For each uncertainty level, we generate $M=5$ independent samples and compute $V(\tilde{s}_k)$ as the average of the value estimates across these samples to reduce estimation variance.
\paragraph{Estimation of Choquet Integrals.}
To approximate the Choquet integrals used in the fuzzy Bellman operator, we leverage the globally learned fuzzy measures. For the standard Choquet integral applied to reward value aggregation, we sort the perturbed value estimates in descending order: $V(\tilde{s}_{(1)}) \geq V(\tilde{s}_{(2)}) \geq \cdots \geq V(\tilde{s}_{(K)})$, where $(i)$ denotes the index after sorting. The corresponding fuzzy measures are computed as $m_i = m(\{(i), (i+1), \ldots, (K)\})$, representing the capacity of the tail sets. Following the discrete Choquet integral formulation, the robust reward value is approximated as:
\begin{equation}
\begin{aligned}
    (C) \int_{\mathcal{P}_s^a} \mathop{\mathbb{E}}_{s' \sim p} [V(s')] \, dm(p) \approx \sum_{i=1}^K V(\tilde{s}_{(i)}) \big(m_i - m_{i+1}\big),
\end{aligned}
\label{eq:standard_choquet}
\end{equation}
where $m_{K+1} = 0$ by convention. For the dual Choquet integral used in cost value aggregation with pessimistic estimation, we sort the cost values in ascending order: $V_c(\tilde{s}_{(1)}) \leq V_c(\tilde{s}_{(2)}) \leq \cdots \leq V_c(\tilde{s}_{(K)})$, and compute:
\begin{equation}
\begin{aligned}
    (C) \int_{\mathcal{P}_s^a} \mathop{\mathbb{E}}_{s' \sim p} [V_c(s')] \, dm'(p) \approx \sum_{i=1}^K V_c(\tilde{s}_{(i)}) \big(m_i - m_{i+1}\big),
\end{aligned}
\label{eq:dual_choquet}
\end{equation}
where $m_i = m(\{1, 2, \ldots, i\})$ represents the capacity of the head sets, computed using the same global fuzzy measure densities $g_k$ but with different subset selection to achieve pessimistic estimation for costs.
The choice of descending order sorting for rewards and ascending order sorting for costs is theoretically grounded in the dual relationship between $m$ and its dual measure $m^*$.

\paragraph{Value Network Updates.}
The value networks are updated through temporal difference learning with Choquet-integrated targets:
\begin{equation}
\begin{aligned}
\mathcal{L}(\theta_r) &= \mathbb{E}_{\tau}\bigl[\bigl(V_{\theta_r}(s_t) - \bigl(r_t + \gamma \cdot \hat{V}_{\theta_r}(s_{t+1})\bigr)\bigr)^2\bigr],\\
\mathcal{L}(\theta_c) &= \mathbb{E}_{\tau}\bigl[\bigl(V_{\theta_c}(s_t) - \bigl(c_t + \gamma \cdot \hat{V}_{\theta_c}(s_{t+1})\bigr)\bigr)^2\bigr],
\end{aligned}
\label{eq:value_updates}
\end{equation}
where $\hat{V}_{\theta_r}$ and $\hat{V}_{\theta_c}$ denote the Choquet-integrated value estimates computed using Equations~\eqref{eq:standard_choquet} and~\eqref{eq:dual_choquet}, respectively.

\paragraph{Fuzzy Network Updates.}
The fuzzy network parameters $\phi$ are updated at a lower frequency than the value networks to ensure stable learning of the uncertainty structure. Given the fuzzy density parameters $g(s_{t+1})$ predicted from next states, the network minimizes the discrepancy between Choquet-integrated predictions and Monte Carlo targets:
\begin{equation}
    L(\phi) = \mathbb{E}_{\tau^{\prime}}\!\Bigl[\bigl(r_t + \gamma \cdot \hat{V}_{\theta_r}(s_{t+1}) - R_t\bigr)^2 +  \bigl(c_t + \gamma \cdot \hat{V}_{\theta_c}(s_{t+1})- C_t\bigr)^2\Bigr],
    \label{eq:loss_fuzzy}
\end{equation}
where $R_t, C_t$ are Monte Carlo returns of reward $r_t$ and cost $c_t$. 

\paragraph{Policy Network Updates.}
Given that fuzzy value estimation implicitly addresses robustness through Choquet integration over multiple perturbations, the robust CMDP problem is solved using a primal-dual approach:
\begin{equation}
    \max_\pi\min_{\lambda_\pi\ge 0} \mathbb{E}_{s\sim d^\pi}\bigl[V_r(s) - \lambda_\pi\bigl(V_c(s)-B\bigr)\bigr],
\label{eq:primal_dual}
\end{equation}
where $V_r$ and $V_c$ represent the robust value estimates obtained through Choquet integration. The policy parameters are optimized to maximize the Lagrangian objective, while the Lagrange multiplier is adjusted to enforce the safety constraint, with specific update rules determined by the underlying safe RL algorithm framework.

\section{Experiments}
\label{main:exp}
\subsection{Experiments Setup}
To fully evaluate the robustness of Fuz-RL in multi-source uncertainties, we conduct experiments on four Safe-Control-Gym \cite{brunke2021safe} tasks: \texttt{Cartpole-Stab}, \texttt{Cartpole-Track}, \texttt{Quadrotor-Stab}, and \texttt{Quadrotor-Track}, as well as four safety-critical control tasks with larger state-action spaces from Safety-Gymnasium \cite{ji2023safety}: \texttt{Safety-PointGoal1-v0}, \texttt{Safety-PointButton1-v0}, \texttt{Safety-PointCircle1-v0}, and \texttt{Safety-PointPush1-v0}.

\textbf{Uncertainty Setting.} During the test phase, we leverage different perturbations provided by the Safe-Control-Gym to consider the following settings on observation, action, and dynamics: 

(1) \textit{Observation uncertainty}. 
We introduce white noise following a normal distribution $\varepsilon \cdot \mathcal{N}(0, I)$ as observation perturbation, where $\varepsilon$ is an adjustable parameter used to set different perturbation intensities. During testing, we vary $\varepsilon$ from $-1$ to $1$ in increments of $0.1$ for Safe-Control-Gym tasks.

(2) \textit{Action uncertainty.} We simulate action uncertainty through an impulse noise disturbance model. Specifically, the perturbed action $\bar{a}_t$ is formulated as $\bar{a}_t = a_t + d_t$, where $d_t$ is defined as:
\begin{equation}
d_t = \begin{cases}
\varepsilon M & t \in [t_{\text{start}}, t_{\text{start}} + D] \\
\varepsilon M\gamma^{(t-t_{\text{start}}-D)} & t > t_{\text{start}} + D \\
0 & \text{otherwise}
\end{cases}
\end{equation}
where $\varepsilon \in [-0.1, 0.1]$ is the magnitude coefficient, $M=10$ is the amplification factor, $t_{start}=20$ is the start step, $D=80$ is the duration, and $\gamma=0.9$ is the decay rate.

(3) \textit{Dynamics uncertainty}. We apply white noise following a normal distribution $\varepsilon \cdot \mathcal{N}(0, I)$ to dynamics parameters—such as pole length and quadrotor mass—where the variation of $\varepsilon$ follows the same scheme as that used for observation perturbation.

(4) \textit{Multi-source uncertainty}. To analyze the non-additive nature of uncertainty perturbations, we simultaneously apply all three uncertainty settings in a coupled configuration during both training and testing. For Safety-Gymnasium tasks, we apply relatively small environmental perturbations during training with $\varepsilon=0.5$. For Fuz-series algorithms, we uniformly adopt the training configuration with $\varepsilon_{\text{base}}=0.1$, $K=10$, and $M=5$.

Since the Safety-Gymnasium benchmarks do not provide built-in uncertainty interfaces, we take observation uncertainty with $\varepsilon \in [0, 0.5]$ as an example for testing.

\textbf{Baselines.} We adopt three safe RL algorithms as the baseline, including the PPO-Lagrangian (PPOL) \cite{ray2019benchmarking}, Conservative Update Policy (CUP) \cite{yang2022cup}, and CVaR-Proximal-Policy-Optimization (CPPO) \cite{ying2021towards}. After integrating the proposed fuzzy-guided framework, we obtain the corresponding Fuz-PPOL, Fuz-CUP, and Fuz-CPPO algorithms. Besides, the current state-of-the-art, robust safe RL, Risk-Averse Model Uncertainty (RAMU) \cite{queeney2024risk}, has also been migrated to our test tasks. All codes of Fuz-RL are implemented based on the SpinningUp \cite{achiam2018spinning}.

\begin{figure}[htbp]
  \centering
  \includegraphics[width=\columnwidth]{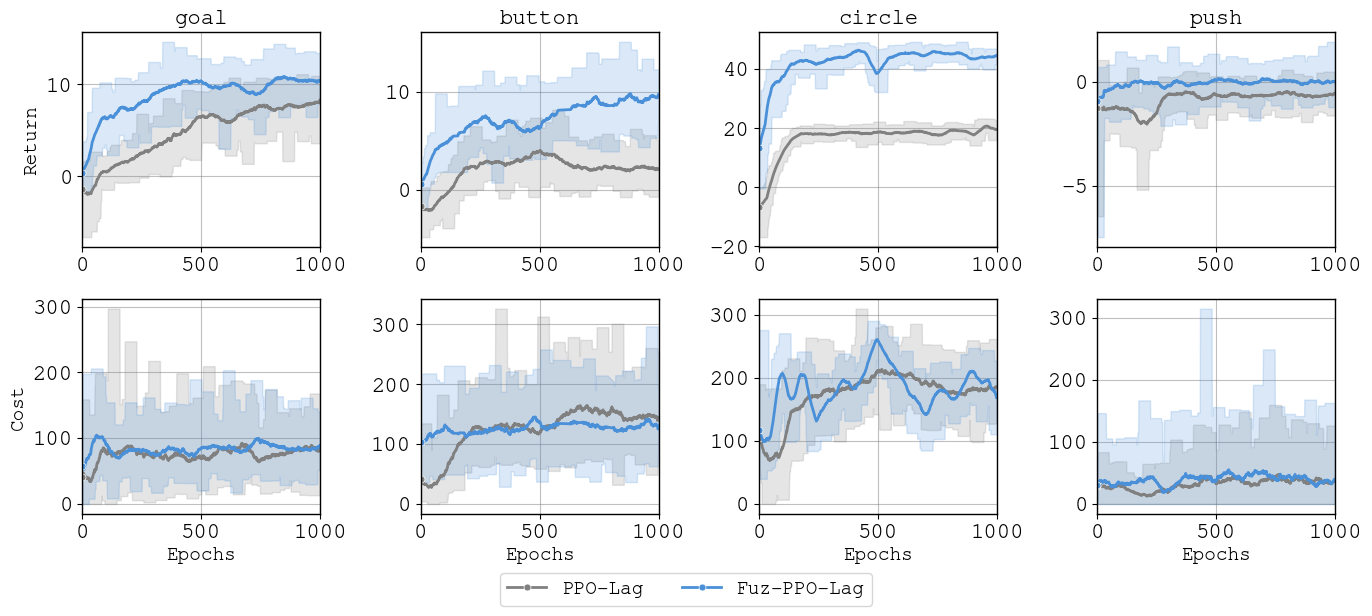} 
\caption{\textbf{Training Dynamics} of PPOLag and Fuz-PPOLag under \texttt{multi-source} uncertainty on Safety-Gymnasium tasks. The perturbation intensity during training is set to $\varepsilon = 0.5$.}
  \label{fig::more_comp}
\end{figure}

\begin{figure}[htbp]
  \centering
  \includegraphics[width=\columnwidth]{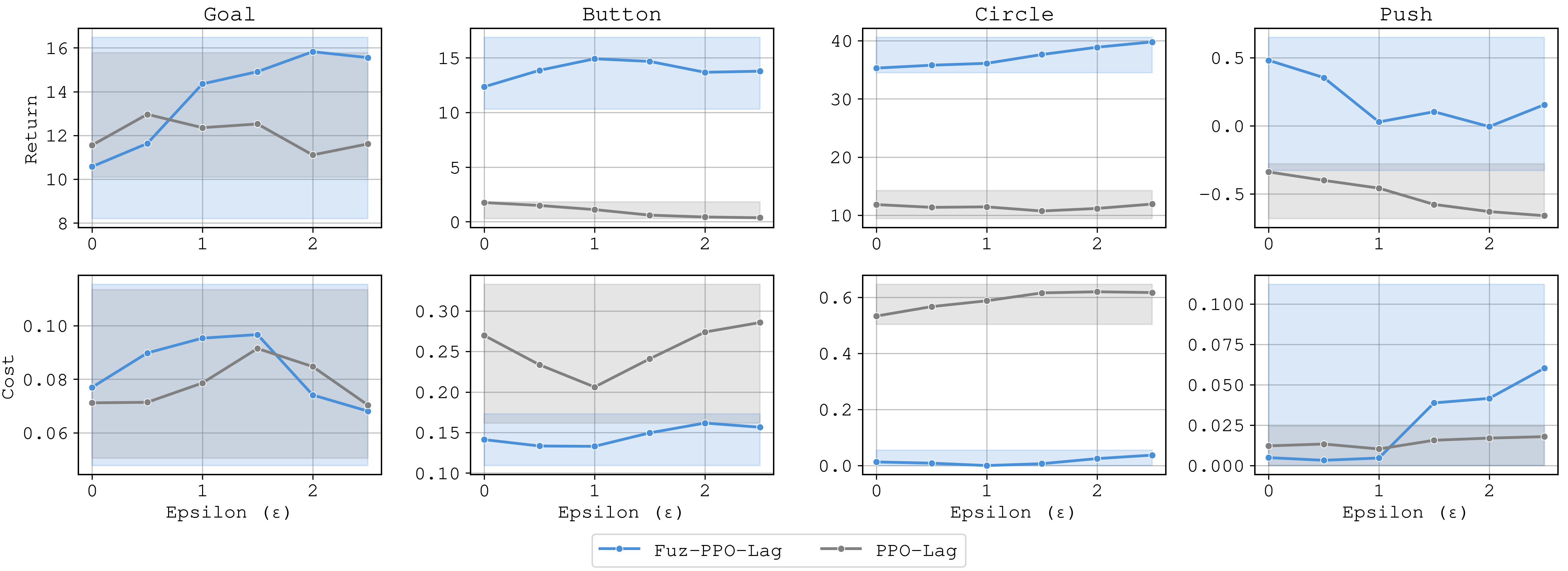} 
  \caption{\textbf{Test Comparison} of PPOLag and Fuz-PPOLag under \texttt{multi-source} uncertainty setting over 5 Episodes and 5 seeds on Safety-Gymnasium tasks. The \texttt{cost\_limit} is set to 0.1.}
  \label{fig::test_point}
\end{figure}

\begin{table}[t]
\caption{\textbf{Detailed evaluation of Safe RL, Fuz-RL, and RAMU }on Safe-Control-Gym tasks with observation, action, dynamics uncertainty. Each value is reported as mean ± standard deviation for 10 episodes and 10 seeds. We shadow the highest AvgRet and lowest AvgRisk for each task.}
\centering
\small
\label{tab_1}
\resizebox{\textwidth}{!}{
\begin{tabular}{@{}cc|cc|cc|cc@{}}
\toprule
\multirow{2}{*}{\textbf{Tasks}} &
  \multirow{2}{*}{\textbf{Methods}} &
  \multicolumn{2}{c|}{\textbf{Observation Uncertainty}} &
  \multicolumn{2}{c|}{\textbf{Action Uncertainty}} &
  \multicolumn{2}{c}{\textbf{Dynamics Uncertainty}} \\ \cmidrule(l){3-8} 
 &
   &
  \multicolumn{1}{c|}{\textbf{AvgRet $\uparrow$}} &
  \textbf{AvgRisk $\downarrow$}&
  \multicolumn{1}{c|}{\textbf{AvgRet $\uparrow$}} &
  \textbf{AvgRisk $\downarrow$}&
  \multicolumn{1}{c|}{\textbf{AvgRet $\uparrow$}} &
  \textbf{AvgRisk $\downarrow$} \\ \midrule
\multicolumn{1}{c|}{\multirow{8}{*}{\shortstack{\textbf{\textit{Cartpole}}\\\textbf{\textit{Stab}}}}} &
  PPOL &
  45 ± 28 &
  0.40 ± 0.14 &
  99 ± 19 &
  0.27 ± 0.14 &
  87 ± 16 &
  0.27 ± 0.09 \\
\multicolumn{1}{c|}{} &
  CUP &
  32 ± 26 &
  0.29 ± 0.11 &
  42 ± 18 &
  0.30 ± 0.13 &
  50 ± 14 &
  0.30 ± 0.13 \\
\multicolumn{1}{c|}{} &
  CPPO &
  41 ± 19 &
  0.47 ± 0.10 &
  76 ± 19 &
  0.36 ± 0.14 &
  77 ± 14 &
  0.31 ± 0.11 \\ \cmidrule(l){2-8} 
\multicolumn{1}{c|}{} &
  RAMU  &
  35 ± 22 &
  0.49 ± 0.16 &
  86 ± 10 &
  0.28 ± 0.07 &
  86 ± 13 &
  \cellcolor{white} \textbf{0.20 ± 0.01} \\ \cmidrule(l){2-8} 
\multicolumn{1}{c|}{} &
  Fuz-PPOL &
  47 ± 25 \textcolor{deepred}{$\uparrow$ 2}&
  0.34 ± 0.11 \textcolor{deepgreen}{$\downarrow$ 0.06}&
  102 ± 17 \textcolor{deepred}{$\uparrow$ 3}&
  0.24 ± 0.12 \textcolor{deepgreen}{$\downarrow$ 0.03}&   
  93 ± 13 \textcolor{deepred}{$\uparrow$ 6}&
  0.22 ± 0.08 \textcolor{deepgreen}{$\downarrow$ 0.05}\\
\multicolumn{1}{c|}{} &
  Fuz-CUP &
  40 ± 19 \textcolor{deepred}{$\uparrow$ 8}&
  \cellcolor{white} \textbf{0.26 ± 0.09} \textcolor{deepgreen}{$\downarrow$ 0.03}&
  87 ± 17 \textcolor{deepred}{$\uparrow$ 45}&
  \cellcolor{white} \textbf{0.23 ± 0.13} \textcolor{deepgreen}{$\downarrow$ 0.07}&
  74 ± 13 \textcolor{deepred}{$\uparrow$ 24}&
  0.25 ± 0.08 \textcolor{deepgreen}{$\downarrow$ 0.05}\\
\multicolumn{1}{c|}{} &
  Fuz-CPPO &
  \cellcolor{white} \textbf{59 ± 25} \textcolor{deepred}{$\uparrow$ 18}&
  0.32 ± 0.10 \textcolor{deepgreen}{$\downarrow$ 0.15}&
  98 ± 17 \textcolor{deepred}{$\uparrow$ 22}&
  0.26 ± 0.13 \textcolor{deepgreen}{$\downarrow$ 0.10}&
  82 ± 14 \textcolor{deepred}{$\uparrow$ 5}&
  0.28 ± 0.09 \textcolor{deepgreen}{$\downarrow$ 0.03}\\ \midrule\midrule
\multicolumn{1}{c|}{\multirow{8}{*}{\shortstack{\textbf{\textit{Cartpole}}\\\textbf{\textit{Track}}}}} &
  PPOL &
  70 ± 16 &
  0.35 ± 0.14 &
  95 ± 12 &
  0.18 ± 0.10 &
  91 ± 12 &
  0.21 ± 0.10 \\
\multicolumn{1}{c|}{} &
  CUP &
  59 ± 12 &
  0.28 ± 0.11 &
  73 ± 9 &
  0.20 ± 0.10 &
  77 ± 11 &
  0.18 ± 0.08 \\
\multicolumn{1}{c|}{} &
  CPPO &
  87 ± 17 &
  0.42 ± 0.11 &
  113 ± 16 &
  0.26 ± 0.11 &
  106 ± 12 &
  0.32 ± 0.09 \\ \cmidrule(l){2-8} 
\multicolumn{1}{c|}{} &
  RAMU  &
  67 ± 31 &
  \cellcolor{white} \textbf{0.22 ± 0.13} &
  70 ± 34 &
  0.18 ± 0.10 &
  79 ± 31 &
  0.14 ± 0.07 \\ \cmidrule(l){2-8} 
\multicolumn{1}{c|}{} &
  Fuz-PPOL &
  91 ± 20 \textcolor{deepred}{$\uparrow$ 21}&
  0.38 ± 0.16 \textcolor{deepred}{$\uparrow$ 0.03}&
  \cellcolor{white} \textbf{120 ± 14} \textcolor{deepred}{$\uparrow$ 25}&
  \cellcolor{white} \textbf{0.18 ± 0.09} \textcolor{deepgreen}{$\downarrow$ 0.00}&
  \cellcolor{white} \textbf{112 ± 14} \textcolor{deepred}{$\uparrow$ 21}&
  0.22 ± 0.09 \textcolor{deepred}{$\uparrow$ 0.01}\\
\multicolumn{1}{c|}{} &
  Fuz-CUP &
  61 ± 31 \textcolor{deepred}{$\uparrow$ 2}&
  0.24 ± 0.16 \textcolor{deepgreen}{$\downarrow$ 0.04}&
  106 ± 16 \textcolor{deepred}{$\uparrow$ 33}&
  0.18 ± 0.11 \textcolor{deepgreen}{$\downarrow$ 0.02}&
  100 ± 12 \textcolor{deepred}{$\uparrow$ 23}&
  \cellcolor{white} \textbf{0.14 ± 0.07} \textcolor{deepgreen}{$\downarrow$ 0.04}\\
\multicolumn{1}{c|}{} &
  Fuz-CPPO &
  \cellcolor{white} \textbf{93 ± 10} \textcolor{deepred}{$\uparrow$ 6}&
  0.31 ± 0.10 \textcolor{deepgreen}{$\downarrow$ 0.11}&
  107 ± 14 \textcolor{deepgreen}{$\downarrow$ 6}&
  0.21 ± 0.09 \textcolor{deepgreen}{$\downarrow$ 0.05}&
  107 ± 12 \textcolor{deepred}{$\uparrow$ 1}&
  0.23 ± 0.08 \textcolor{deepgreen}{$\downarrow$ 0.09}\\ \midrule\midrule
  \multicolumn{1}{c|}{\multirow{8}{*}{\shortstack{\textbf{\textit{Quadrotor}}\\\textbf{\textit{Stab}}}}} &
  PPOL &
  \cellcolor{white} \textbf{164 ± 18} &
  0.13 ± 0.06 &
  58 ± 55 &
  0.52 ± 0.19 &
  142 ± 34 &
  0.28 ± 0.11 \\
\multicolumn{1}{c|}{} &
  CUP &
  139 ± 7 &
  0.05 ± 0.02 &
  58 ± 28 &
  0.56 ± 0.17 &
  117 ± 26 &
  0.14 ± 0.10 \\
\multicolumn{1}{c|}{} &
  CPPO &
  131 ± 14 &
  0.09 ± 0.02 &
  54 ± 50 &
  0.36 ± 0.11 &
  117 ± 36 & 
  0.17 ± 0.13 \\ \cmidrule(l){2-8} 
\multicolumn{1}{c|}{} &
  RAMU  &
  146 ± 16 &
  0.06 ± 0.04 &
  28 ± 33 &
  0.59 ± 0.20 &
  120 ± 39 &
  0.17 ± 0.16 \\ \cmidrule(l){2-8} 
\multicolumn{1}{c|}{} &
  Fuz-PPOL &
  161 ± 22 \textcolor{deepgreen}{$\downarrow$ 3}&
  0.07 ± 0.05 \textcolor{deepgreen}{$\downarrow$ 0.06}&
  67 ± 58 \textcolor{deepred}{$\uparrow$ 9}&
  0.43 ± 0.19 \textcolor{deepgreen}{$\downarrow$ 0.09}&
  \cellcolor{white} \textbf{156 ± 28} \textcolor{deepred}{$\uparrow$ 14}&
  0.13 ± 0.09 \textcolor{deepgreen}{$\downarrow$ 0.15}\\
\multicolumn{1}{c|}{} &
  Fuz-CUP &
  142 ± 15 \textcolor{deepred}{$\uparrow$ 3}&
  \cellcolor{white} \textbf{0.05 ± 0.02} \textcolor{deepgreen}{$\downarrow$ 0.00}&
  \cellcolor{white} \textbf{94 ± 24} \textcolor{deepred}{$\uparrow$ 36}&
  \cellcolor{white} \textbf{0.39 ± 0.10} \textcolor{deepgreen}{$\downarrow$ 0.17}&
  139 ± 23 \textcolor{deepred}{$\uparrow$ 22}&
  0.14 ± 0.09 \textcolor{deepgreen}{$\downarrow$ 0.00}\\
\multicolumn{1}{c|}{} &
  Fuz-CPPO &
  156 ± 11 \textcolor{deepred}{$\uparrow$ 25}&
  0.07 ± 0.03 \textcolor{deepgreen}{$\downarrow$ 0.02}&
  87 ± 31 \textcolor{deepred}{$\uparrow$ 33}&
  0.33 ± 0.10 \textcolor{deepgreen}{$\downarrow$ 0.03}&
  130 ± 44 \textcolor{deepred}{$\uparrow$ 13}&
  \cellcolor{white} \textbf{0.09 ± 0.12} \textcolor{deepgreen}{$\downarrow$ 0.08}\\ \midrule\midrule
\multicolumn{1}{c|}{\multirow{8}{*}{\shortstack{\textbf{\textit{Quadrotor}}\\\textbf{\textit{Track}}}}} &
  PPOL & 
  \cellcolor{white} \textbf{218 ± 8} & 
  0.48 ± 0.04 & 
  104 ± 79 & 
  0.81 ± 0.12 & 
  \cellcolor{white} \textbf{203 ± 24} & 
  0.58 ± 0.12 \\
\multicolumn{1}{c|}{} &
  CUP &
  151 ± 14 &
  0.04 ± 0.03 &
  67 ± 50 &
  0.37 ± 0.13 &
  152 ± 13 &
  0.12 ± 0.11 \\
\multicolumn{1}{c|}{} &
  CPPO &
  152 ± 16 &
  0.77 ± 0.04 &
  76 ± 60 &
  0.72 ± 0.11 &
  124 ± 33 &
  0.73 ± 0.08 \\ \cmidrule(l){2-8} 
\multicolumn{1}{c|}{} &
  RAMU  &
  176 ± 12 &
  0.05 ± 0.03 &
  61 ± 50 &
  0.53 ± 0.18 &
  123 ± 48 &
  0.31 ± 0.20 \\ \cmidrule(l){2-8} 
\multicolumn{1}{c|}{} &
  Fuz-PPOL &
  200 ± 6 \textcolor{deepgreen}{$\downarrow$ 18}&
  0.28 ± 0.05 \textcolor{deepgreen}{$\downarrow$ 0.20}&
  99 ± 67 \textcolor{deepgreen}{$\downarrow$ 5}&
  0.64 ± 0.17 \textcolor{deepgreen}{$\downarrow$ 0.17}&
  194 ± 16 \textcolor{deepgreen}{$\downarrow$ 9}&
  0.38 ± 0.15 \textcolor{deepgreen}{$\downarrow$ 0.20}\\
\multicolumn{1}{c|}{} &
  Fuz-CUP &
  175 ± 9 \textcolor{deepred}{$\uparrow$ 24}&
  \cellcolor{white} \textbf{0.04 ± 0.02} \textcolor{deepgreen}{$\downarrow$ 0.00}&
  \cellcolor{white} \textbf{112 ± 22} \textcolor{deepred}{$\uparrow$ 45}&
  \cellcolor{white} \textbf{0.33 ± 0.09} \textcolor{deepgreen}{$\downarrow$ 0.04}&
  168 ± 13 \textcolor{deepred}{$\uparrow$ 16}&
  \cellcolor{white} \textbf{0.12 ± 0.10} \textcolor{deepgreen}{$\downarrow$ 0.00}\\
\multicolumn{1}{c|}{} &
  Fuz-CPPO &
  168 ± 14 \textcolor{deepred}{$\uparrow$ 16}&
  0.47 ± 0.05 \textcolor{deepgreen}{$\downarrow$ 0.30}&
  79 ± 53 \textcolor{deepred}{$\uparrow$ 3}&
  0.67 ± 0.09 \textcolor{deepgreen}{$\downarrow$ 0.05}&
  151 ± 23 \textcolor{deepred}{$\uparrow$ 27}&
  0.59 ± 0.12 \textcolor{deepgreen}{$\downarrow$ 0.14}\\ \bottomrule
\end{tabular}
}
\end{table}

\subsection{Robustness Assessment in Safe Control Tasks}
For Safe-Control-Gym tasks, we trained the three safe RL baseline algorithms along with their corresponding Fuz-RL and RAMU variants under the same configuration. The specific parameter settings and more detailed results are presented in Appendix \ref{appendix:Hyper_parameters}. For Safety-Gymnasium tasks, we use PPOLag and Fuz-PPOLag as representative examples for experimental evaluation, with training and testing dynamics shown in Figure~\ref{fig::more_comp} and Figure~\ref{fig::test_point}, respectively.

We evaluate the models from two features ``\textit{AvgRet}'' and ``\textit{AvgRisk}'', which represent the average episodic return and the \textit{proportion of constraint violations}, respectively, for each task over 10 episodes across 10 seeds. 

\paragraph{Comparison between Safe RL and Fuz-RL.} As depicted in Table (\ref{tab_1}) and Appendix \ref{appendix:results}, Fuz-RL demonstrates superior safety in \textbf{94.9\%} cases and robust performance in \textbf{88.9\%} tasks across various uncertainty settings. Taking Fuz-CUP as an example, it achieves 61.4\% higher \textit{AvgRet} and 16.7\% lower \textit{AvgRisk} than CUP in CartPole-Stab task. Moreover, Fuz-RL shows better uncertainty resistance with slower performance degradation than Safe RL. The variance reduction in \textit{AvgRet} is 20.7\%, 9.9\%, and 8.6\%, while in \textit{AvgRisk} is 13.2\%, 7.1\%, and 22.6\% respectively.

\paragraph{Comparison between Fuz-RL and RAMU.} In the 36 Fuz-RL-based experiments listed in Table (\ref{tab_1}), Fuz-RL surpasses RAMU in achieving higher \textit{AvgRet} in \textbf{83.3\% }of the tasks. Furthermore, Fuz-RL exhibits lower \textit{AvgRisk} compared to RAMU in 52.8\% of them. It is important to highlight that \textit{the lower average episodic risk} of RAMU is achieved by \textit{compromising average episodic rewards}, especially in cases of actions affected by impulse disturbances, as shown in the ``Action Uncertainty'' section of Table (\ref{tab_1}). Fuz-RL consistently outperforms RAMU in all ``Action Uncertainty'' tasks.

\begin{figure}[t]
  \centering
  \includegraphics[width=\columnwidth]{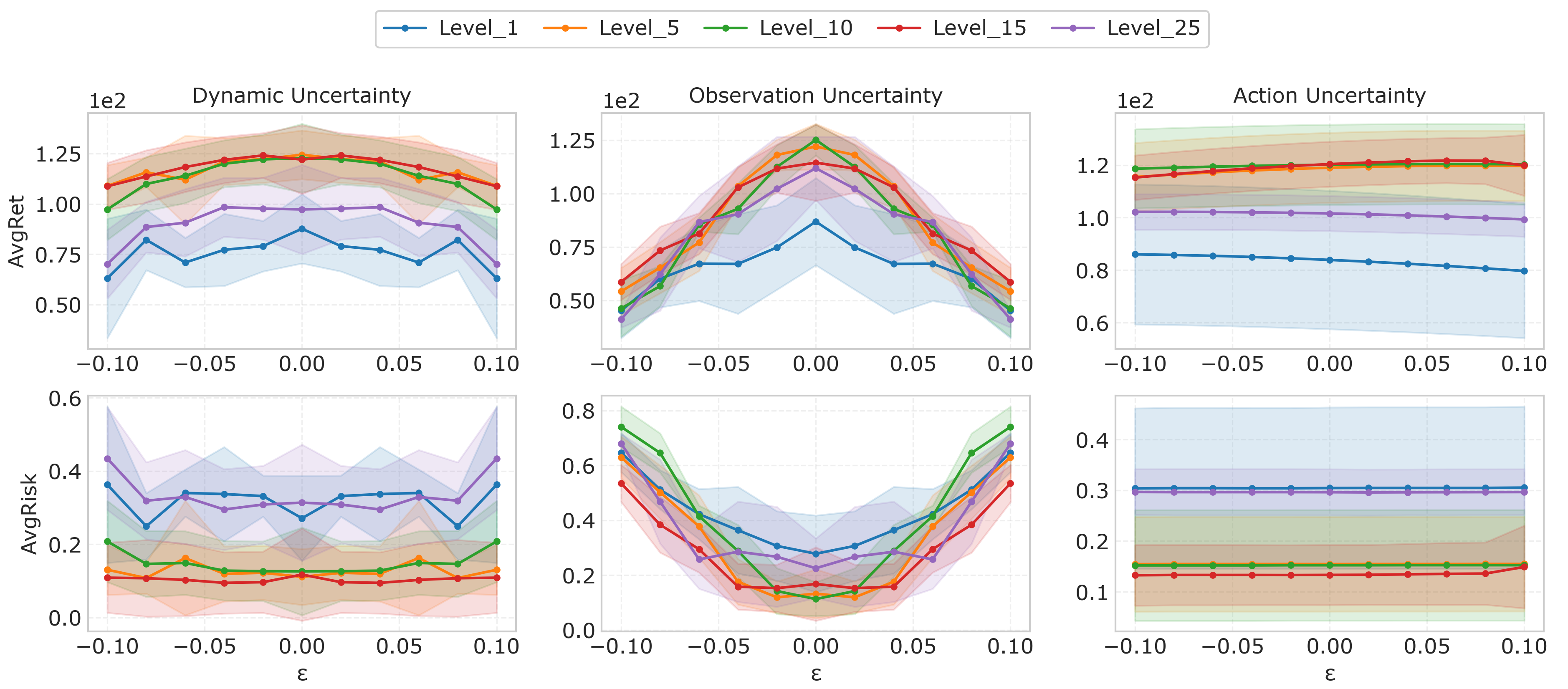} 
  \caption{Ablation study of the uncertainty level $K$.}
  \label{fig:ablation_level}
\end{figure}

\paragraph{Ablation Studies of Fuz-RL.} We conduct ablation studies to examine how different uncertainty levels $K$ affect Fuz-RL's performance. As illustrated in Figure \ref{fig:ablation_level}, setting uncertainty level ($K=1$) proves insufficient and leads to high episodic risk, while excessive levels ($K=25$) complicate training and reduce rewards. The optimal performance emerges at intermediate levels ($K=5$ to $K=15$), where agents achieve higher rewards while maintaining lower and more stable risk across all three uncertainty types.

\section{Conclusion and Future Work} \label{conclu}
In this paper, we propose Fuz-RL, a novel robustness enhancement framework that seamlessly integrates fuzzy logic into safe reinforcement learning. We develop a novel fuzzy Bellman operator incorporating Choquet integrals, enabling robust decision-making without solving computationally expensive min-max optimization problems. Theoretically, we establish the equivalence between our fuzzy robust safe RL formulation and distributionally robust safe RL. Extensive experiments on the safe-control-gym and safety-gymnasium benchmarks demonstrate that Fuz-RL significantly outperforms state-of-the-art safe and robust RL algorithms across various uncertainty types, achieving superior performance in both reward optimization and safety constraint satisfaction under diverse perturbation scenarios.

While Fuz-RL demonstrates promising results, it faces limited scalability in high-dimensional state spaces. Future work will focus on developing more efficient uncertainty modeling techniques and extending the framework to handle non-stationary uncertainty distributions through adaptive learning mechanisms.

\section*{Acknowledgements}
This work was supported in part by the Smart Grid-National Science and Technology Major Project (2025ZD0803600, 2025ZD0803604), the National Natural Science Foundation of China under Grants 72571007, and the Natural Science Foundation of Zhejiang Province under Grant LZ23F030009.
\bibliographystyle{plain}
\bibliography{example_paper}

\newpage
\appendix

\section{Appendix / Theorems and Proofs}
\label{app:proofs}

\begin{theorem}[$\gamma$-contraction of Fuzzy Bellman Operator]
\label{app_thm:gamma_contraction}
For any $V_1, V_2 \in \mathcal{B}(\mathcal{S})$,
\[
\|\mathcal{F}(V_1) - \mathcal{F}(V_2)\|_{\infty} \leq \gamma \|V_1 - V_2\|_{\infty}.
\]
\end{theorem}

\begin{proof}
For any two value functions $V_1$ and $V_2$, and any state $s$:
\begin{align*}
& |\mathcal{F}(V_1)(s) - \mathcal{F}(V_2)(s)| \\
&= \left|\mathbb{E}_{a\sim \pi} \left[\gamma (C)\int_{\mathcal{P}_s^a} \mathbb{E}_{s' \sim p} [V_1(s') - V_2(s')] dm(p)\right]\right| \\
&\leq \gamma \mathbb{E}_{a\sim \pi} \left[(C)\int_{\mathcal{P}_s^a} \mathbb{E}_{s' \sim p} |V_1(s') - V_2(s')| dm(p)\right] \\
&\leq \gamma \mathbb{E}_{a\sim \pi} \left[\|V_1 - V_2\|_{\infty} (C)\int_{\mathcal{P}_s^a} dm(p)\right] \\
&= \gamma \|V_1 - V_2\|_{\infty} \mathbb{E}_{a\sim \pi} \left[(C)\int_{\mathcal{P}_s^a} dm(p)\right] \\
&= \gamma \|V_1 - V_2\|_{\infty}
\end{align*}
Here, we use the fact that $(C)\int_{\mathcal{P}_s^a} dm(p) = 1$ for all $s$ and $a$, as $m$ is a normalized fuzzy measure. Taking the supremum over all states $s$ yields the result.
\end{proof}

\begin{theorem}[Convergence of Fuzzy Bellman Operator]
\label{app_thm:convergence}
Let $V^0 \in \mathcal{B}(\mathcal{S})$ be an initial value function and $V^{n+1} = \mathcal{F}(V^n)$. Then $V^n$ converges to a unique fixed point $V^*$ satisfying $V^*=\mathcal{F}(V^*)$ with geometric rate $\|V^n - V^*\|_\infty \leq \gamma^n \|V^0 - V^*\|_\infty$.
\end{theorem}

\begin{proof}
By Theorem \ref{thm:gamma_contraction}, $\mathcal{F}$ is a contraction mapping. The Banach Fixed Point Theorem guarantees the existence of a unique fixed point $V^*$ such that $V^* = \mathcal{F}(V^*)$. Moreover, for any initial $V^0$, the sequence $\{V^n\}_{n=0}^{\infty}$ defined by $V^{n+1} = \mathcal{F}(V^n)$ converges to $V^*$:
\begin{equation*}
\|V^n - V^*\|_{\infty} \leq \gamma^n \|V^0 - V^*\|_{\infty} \to 0 \text{ as } n \to \infty
\end{equation*}
This convergence follows directly from the contraction property:
\begin{align*}
\|V^{n+1} - V^*\|_{\infty} &= \|\mathcal{F}(V^n) - \mathcal{F}(V^*)\|_{\infty} \\
&\leq \gamma \|V^n - V^*\|_{\infty} \\
&\leq \gamma^n \|V^1 - V^*\|_{\infty} \\
&\leq \gamma^n \|V^1 - V^0\|_{\infty} + \gamma^n \|V^0 - V^*\|_{\infty}
\end{align*}
As $n \to \infty$, both terms approach zero due to $\gamma < 1$, proving the convergence.
\end{proof}

\begin{lemma}[Core Duality of Convex Fuzzy Measures]
\label{lem:core_duality}
Let $m$ be a convex fuzzy measure on $\mathcal{P}_s^a$ with dual measure defined by:
\[
m'(A) := 1 - m(\mathcal{P}_s^a \setminus A), \quad \forall A \subseteq \mathcal{P}_s^a.
\]
Then the cores satisfy $\mathrm{core}(m') = \mathrm{core}(m)$, and for any bounded measurable function $f: \mathcal{P}_s^a \to \mathbb{R}$:
\[
(C)\int_{\mathcal{P}_s^a} f \, dm' = \max_{P \in \mathrm{core}(m)} \mathbb{E}_P[f].
\]
\end{lemma}

\begin{proof}
\textbf{Part 1: Core Equivalence}.  
For any convex fuzzy measure $m$, its dual $m'$ is concave. By the duality theorem for balanced fuzzy measures \cite{grabisch2016set}:
\begin{align*}
P \in \mathrm{core}(m) &\iff P(A) \geq m(A),\ \forall A \subseteq \mathcal{P}_s^a \\
&\iff 1 - P(\mathcal{P}_s^a \setminus A) \geq 1 - m(\mathcal{P}_s^a \setminus A),\ \forall A \subseteq \mathcal{P}_s^a \\
&\iff P(A) \geq m'(A),\ \forall A \subseteq \mathcal{P}_s^a \\
&\iff P \in \mathrm{core}(m').
\end{align*}

\textbf{Part 2: Maximum Representation}.  
For the concave measure $m'$, the Choquet integral attains its maximum over the core:
\[
(C)\int f \, dm' = \max_{P \in \mathrm{core}(m')} \mathbb{E}_P[f] = \max_{P \in \mathrm{core}(m)} \mathbb{E}_P[f],
\]
where the last equality follows from $\mathrm{core}(m') = \mathrm{core}(m)$. \qedhere
\end{proof}

\begin{theorem}[Equivalence Theorem]
\label{thm:equiv_dual}
Given a robust CMDP:
\[
\begin{aligned}
    \max_{\pi} \min_{p\in\mathcal{P}} &\mathbb{E}_{\tau\sim(\pi,p)}\left[\sum_{t=0}^{\infty}\gamma^t r(s_t,a_t)\right] \\ 
    \text{s.t.} \quad \max_{p\in\mathcal{P}} &\mathbb{E}_{\tau\sim(\pi,p)}\left[\sum_{t=0}^{\infty}\gamma^t c(s_t,a_t)\right] \leq B,
\end{aligned}
\]
let \( m \) be a convex $\lambda$-fuzzy measure on \( \mathcal{P}_s^a \) such that:
1. \( \text{core}(m) \subseteq \mathcal{P} \),
2. \( \arg\min_{p \in \mathcal{P}} \mathbb{E}[r] \in \text{core}(m) \),
3. \( \arg\max_{p \in \mathcal{P}} \mathbb{E}[c] \in \text{core}(m) \).

Define the dual fuzzy measure \( m'(A) := 1 - m(\mathcal{P}_s^a \setminus A) \). Then the Fuz-RL problem:
\[
\max_{\pi} J_r^{\mathcal{F}}(\pi) \quad \text{s.t.} \quad J_c^{\mathcal{F}}(\pi) \leq B,
\]
where 
\[
\begin{aligned}
  J_r^{\mathcal{F}}(\pi) = \mathbb{E}_{s_0\sim d_0}\left[(C)\int_{\mathcal{P}_s^a} \sum_{t=0}^\infty \gamma^t r(s_t,\pi(s_t)) dm(p)\right], \\ 
  J_c^{\mathcal{F}}(\pi) = \mathbb{E}_{s_0\sim d_0}\left[(C)\int_{\mathcal{P}_s^a} \sum_{t=0}^\infty \gamma^t c(s_t,\pi(s_t)) dm'(p)\right],  
\end{aligned}
\]
is equivalent to the original robust CMDP.
\end{theorem}

\begin{proof}
\textbf{Step 1: Core Inclusion and Extremal Coverage}.  
By the fuzzy measure network's design, each uncertainty level \( \mathcal{P}_{s,k}^a \) is bounded within the \( \epsilon \)-neighborhood of the nominal dynamics (Equation \ref{eq:uncertainty_set}). The sigmoid function ensures:
\[
0 < m(\mathcal{P}_{s,k}^a) < 1,
\]
and assigns non-zero weights to extremal perturbations \( \arg\min_{p \in \mathcal{P}} \mathbb{E}[r] \) and \( \arg\max_{p \in \mathcal{P}} \mathbb{E}[c] \), guaranteeing:
\[
\arg\min_{p \in \mathcal{P}} \mathbb{E}[r] \in \text{core}(m), \quad \arg\max_{p \in \mathcal{P}} \mathbb{E}[c] \in \text{core}(m).
\]
Thus, \( \text{core}(m) \subseteq \mathcal{P} \) and covers extremal points of \( \mathcal{P} \).

\textbf{Step 2: Duality of Fuzzy Measures}.  
For the convex fuzzy measure \( m \), its dual \( m' \) is concave. By Choquet duality \cite{grabisch2016set}:
\[
(C)\int f \, dm = \inf_{q \in \text{core}(m)} \mathbb{E}_q[f], \quad (C)\int f \, dm' = \sup_{q \in \text{core}(m)} \mathbb{E}_q[f].
\]

\textbf{Step 3: Reward Objective Equivalence}.  
For the reward function:
\[
J_r^{\mathcal{F}}(\pi) = \mathbb{E}_{s_0} \left[ (C)\int \mathbb{E}[r] dm \right] \overset{(a)}{=} \mathbb{E}_{s_0} \left[ \min_{q \in \text{core}(m)} \mathbb{E}_q[r] \right] \overset{(b)}{=} \min_{p \in \mathcal{P}} \mathbb{E}_{(\pi,p)}[r],
\]
where (a) uses Choquet minimality for convex \( m \), and (b) holds because \( \text{core}(m) \) contains \( \arg\min_{p \in \mathcal{P}} \mathbb{E}[r] \).

\textbf{Step 4: Cost Constraint Equivalence}.  
For the cost function:
\[
J_c^{\mathcal{F}}(\pi) = \mathbb{E}_{s_0} \left[ (C)\int \mathbb{E}[c] dm' \right] \overset{(c)}{=} \mathbb{E}_{s_0} \left[ \max_{q \in \text{core}(m)} \mathbb{E}_q[c] \right] \overset{(d)}{=} \max_{p \in \mathcal{P}} \mathbb{E}_{(\pi,p)}[c],
\]
where (c) uses Choquet maximality for concave \( m' \), and (d) holds because \( \text{core}(m) \) contains \( \arg\max_{p \in \mathcal{P}} \mathbb{E}[c] \).

\textbf{Step 5: Equivalence Conclusion}.  
Combining Steps 3–4, the Fuz-RL problem:
\[
\max_{\pi} J_r^{\mathcal{F}}(\pi) \quad \text{s.t.} \quad J_c^{\mathcal{F}}(\pi) \leq B
\]
is equivalent to the original robust CMDP, as both objectives and constraints encode the same worst-case expectations over \( \mathcal{P} \). \qedhere
\end{proof}

\section{Appendix/Algorithm}
\label{app::alg}

\begin{algorithm}[htb]
\caption{Fuzzy-Guided Robust Framework for Safe RL (Fuz-RL)}
\label{alg: Fuz_RL}
\begin{algorithmic}[1]
\State \textbf{Input:} actor $\theta_{\pi}$, critics $\theta_r, \theta_c$, fuzzy density parameters $g$, uncertainty levels $K$, replay buffer $\mathcal{D}$.
\State \textbf{Initialize:} $\theta_r, \theta_c, \theta_\pi, g$, buffer $\mathcal{D}$.
\For{epoch $= 1$ to MaxEpoch}
    \For{$t = 1$ to $T$}
        \State Sample action $a_t \sim \pi_{\theta_{\pi}}(\cdot|s_t)$, observe $s_{t+1}$ and get $r_t, c_t$.
        \State For each $i \in \{1,\dots,K\}$, generate perturbed state $\tilde{s}_i = s + \epsilon_i \cdot \mathcal{N}(0,I)$.
        \State Store the tuple $(s_t,a_t,r_t,c_t,s_{t+1}, \{\tilde{s}_i\}_{i=1}^K)$ in $\mathcal{D}$.
    \EndFor
    \For{each \textit{actor/critic network update step}}
        \State Sample mini-batch $\tau$ from $\mathcal{D}$.
        \State Compute perturbed values $\{V_{\theta_r}(\tilde{s}_i)\}$ and $\{V_{\theta_c}(\tilde{s}_i)\}$.
        \State Calculate Choquet integrals using Eqs.~\eqref{eq:standard_choquet}--\eqref{eq:dual_choquet}.
        \State Update $V_{\theta_r}, V_{\theta_c}$, fuzzy density parameters $g$ using Eq.~\eqref{eq:value_updates} and Eq.~\eqref{eq:loss_fuzzy}.
        \State Update $\pi_{\theta_\pi}$ using Eq.~\eqref{eq:primal_dual} with specific safe RL algorithm.
    \EndFor
\EndFor
\State \textbf{Output:} Trained parameters $\theta_\pi, \theta_r, \theta_c, g$.
\end{algorithmic}
\end{algorithm}

\subsection{Optimization Details}
\label{app::optimization}

For optimal policy optimization based on value estimation, we solve the constrained optimization problem using the Lagrangian method:

\begin{equation}
\mathcal{L}(\pi, \lambda) = J_r^{\mathcal{F}}(\pi) - \lambda(J_c^{\mathcal{F}}(\pi) - B)
\end{equation}

The optimal policy $\pi^*$ and the optimal Lagrangian multiplier $\lambda^*$ can be obtained by:
\begin{equation}
(\pi^*, \lambda^*) = \arg\max_{\pi}\min_{\lambda\geq0} \mathcal{L}(\pi, \lambda)
\end{equation}

For a given state $s$, the optimal action selection rule becomes:
\begin{equation}
\pi^*(a|s) = \arg\max_{a\in\mathcal{A}} Q_r^{\pi}(s,a) - \lambda^*Q_c^{\pi}(s,a)
\end{equation}

where the action-value functions are:
\begin{align}
Q_r^{\pi}(s,a) &= r(s,a) + \gamma\mathbb{E}_{s'\sim p}[V_r^*(s')] \\
Q_c^{\pi}(s,a) &= c(s,a) + \gamma\mathbb{E}_{s'\sim p}[V_c^*(s')]
\end{align}

Here, $V_r^*$ and $V_c^*$ are the unique fixed points guaranteed by Theorems \ref{thm:gamma_contraction} and \ref{thm:convergence}, representing the optimal robust value functions for reward and cost, respectively.

In practice, we compute the optimal policy iteratively by initializing a Lagrangian multiplier $\lambda^{(0)}$ and then alternating between policy updates and multiplier updates. At each iteration $k$, we compute:
\begin{align}
\pi^{(k)} &= \arg\max_{\pi} J_r^{\mathcal{F}}(\pi) - \lambda^{(k)}(J_c^{\mathcal{F}}(\pi) - B) \\
\lambda^{(k+1)} &= [\lambda^{(k)} + \alpha(J_c^{\mathcal{F}}(\pi^{(k)}) - B)]^+
\end{align}
where $\alpha > 0$ is a step size and $[x]^+ = \max(0,x)$. This process continues until convergence, yielding the optimal safe policy $\pi^*$ that maximizes reward while satisfying the safety constraint.

\section{Appendix / Experiment Setting and More Results}
\subsection{Environment description}
\label{appendix:Environment_description}
\subsubsection{Double Integrator}
The following dynamics describe the double integrator:
\begin{equation}
\begin{bmatrix}
x_{t+1} \\
v_{t+1}
\end{bmatrix}
= 
\begin{bmatrix}
1 & 0 \\
0 & 1
\end{bmatrix}
\begin{bmatrix}
x_t \\
v_t
\end{bmatrix}
+ 
\begin{bmatrix}
0.005 \\ 0
\end{bmatrix}
a_t,
\end{equation}
where $a_t \in [-1,1]$. The safety constraints are $|x| \leq 2$ and $|v| \leq 2$. 

The reward function induced to unsafe state is designed as follows:

\begin{equation}
\begin{aligned}
    r(x, v) =   &\max(4 - (2(x - 1.5)^2 + 2(v + 1.5)^2), 0) + \\
                &\max(5 - (3(x + 2.2)^2 + 3(v + 2.2)^2), 0) + \\
                &\max(5 - (3(x - 2.2)^2 + 3(v - 2.2)^2), 0) + \\
                &\max(4 - (2(x + 1.5)^2 + 2(v - 1.5)^2), 0)
\end{aligned}
\end{equation}

\subsubsection{Safe Control Gym}
The safe-control-gym benchmark comprises three dynamical systems: the Cartpole, and the 1D and 2D Quadrotors, as shown in Figure. \ref{system}. In our setting, we use CartPole and 2D QuadRotor as the base environments.

\begin{figure}[h]
  \centering
  \includegraphics[width=1\columnwidth]{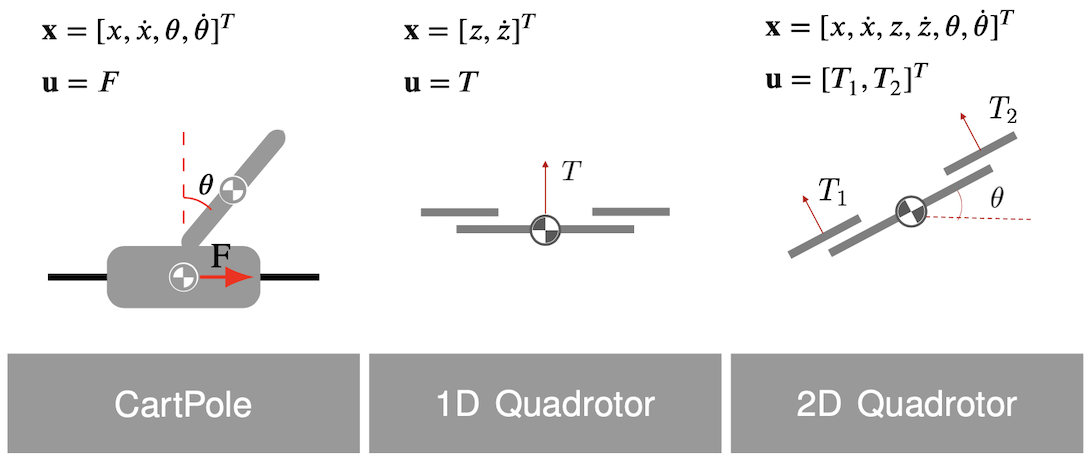} 
  \caption{Schematics, state and input vectors of the cart-pole, and the 1D and 2D quadrotor environments in safe-control-gym.}
  \label{system}
\end{figure}

For the CartPole system, the system state includes position $x$ and velocity $v$ of the cart, angle $\theta$, and angular velocity $\omega$ of the pole. The control inputs $u \in [-1,1] \subset \mathbb{R}$ and external disturbances $a \in [-0.5,0.5] \subset \mathbb{R}$ are horizontal forces applied on the cart. The safety constraints are $|\theta| \leq 0.2$, i.e., keeping the pole nearly upright. The constraint function is $h(\theta) = \min\{\theta + 0.2, 0.2 - \theta\}$. 

For the 2D QuadRotor system, the state of the system is given by $\mathbf{s} = 
\begin{bmatrix}
x, 
\dot{x}, 
z, 
\dot{z}, 
\theta, 
\dot{\theta}
\end{bmatrix}^T$, where $(x, z)$ and $(\dot{x}, \dot{z})$ are the translation position and velocity of the COM of the quadrotor in the $xz$-plane, and $\theta$ and $\dot{\theta}$ are the pitch angle and the pitch angle rate, respectively. The input of the system is the thrusts $\mathbf{a} = \begin{bmatrix} T_1, T_2 \end{bmatrix}^T$ generated by two pairs of motors, one on each side of the body’s $y$-axis. The safety constraints are $z - 0.5 > 0$ and $1.5 - z > 0$, i.e., maintaining its vertical position $z$ between $[-0.5, 1.5]$. The constraint function is $h(z) = \min\{z - 0.5, 1.5 - z\}$.

For the reward function setup, we utilize a weighted sum of the errors between the current state $\mathbf{s}$, action $\mathbf{a}$, and their reference values as the reward for each step. The details of the weighting are provided in Table \ref{params_setting}.

Besides, each environment in Safe-Control-Gym supports two control tasks: stabilization and trajectory tracking. For stabilization, safe-controlgym provides an equilibrium pair for the system, $x^{\text{ref}}, u^{\text{ref}}$. For trajectory tracking, the benchmark includes a trajectory generation module capable of generating circular, sinusoidal, lemniscate, or square trajectories. The module returns references $ x_{\text{ref}_i},\ u_{\text{ref}_i} \ \forall i \in \{0, \ldots, L\} $, where $L$ is the number of control steps in an episode. 

\subsection{Hyper-parameters}
\label{appendix:Hyper_parameters}

\subsubsection{Hyper-parameters of RL}
In all the experiments, we have revised the benchmark algorithms and Fuz-RL employing the RL framework provided by Spinning Up. The complete hyperparameters used in the experiments are shown in Table \ref{params_setting}.

Particularly, for the CPPO and Fuz-CPPO algorithms, the risk threshold $\beta$ for adverse trajectories is set to 100. In the case of the PPOL and CUP algorithms, the initial value of the Lagrange coefficient is set to 0.001, with an upper limit of 0.2 and a learning rate of 0.02. For the RAMU algorithm, the Wang transform is utilized with $\eta= 0.75$, which is applied to both the objective and the constraint.

\begin{table}[ht]
\centering
\caption{Hyperparameter Settings of Fuz-RL Training and Testing}
\label{params_setting}
\begin{tabular}{@{}c|cccc@{}}
\toprule
                   & CartPole-Stab    & CartPole-Track            & QuadRotor-Stab         & QuadRotor-Track                    \\ 
\midrule
rollout length           & 150    & 150    & 250    & 250    \\
training epoch           & 500    & 500    & 1000   & 1000   \\
batch size               & 64     & 64     & 64     & 128    \\
cost limit               & 1      & 1      & 10     & 10     \\
uncertainty level $K$           & 10     & 10     & 15     & 15    \\
optimization step    & 40     & 40     & 80     & 80     \\
actor learning rate      & 0.0003 & 0.0003 & 0.0002 & 0.0002 \\
critic learning rate     & 0.001  & 0.001  & 0.001  & 0.001  \\
fuzzy learning rate      & 0.0003 & 0.0003 & 0.0003 & 0.0003 \\
target KL                & 0.2    & 0.2    & 0.15   & 0.15   \\
hidden\_sizes      & {[}64, 64{]}     & {[}64, 64{]}              & {[}256, 128{]}         & {[}256, 128{]}                     \\
rew\_act\_weight         & 0.1    & 0.01   & 0.1    & 0.01 \\
rew\_state\_weight & $\begin{bmatrix} 1 &1 \\ 1 &1\end{bmatrix}$ & $\begin{bmatrix} 1 & 0.01 \\ 0.01 & 0.01 \end{bmatrix}$ & $\begin{bmatrix} 1  &1 \\ 1 &1 \\ 1 &1\end{bmatrix}$ & $\begin{bmatrix} 1 & 0.01 \\ 1 & 0.01 \\ 0.01 & 0.01 \end{bmatrix}$ \\ 
\bottomrule
\end{tabular}
\end{table}

\subsubsection{Uncertainty setting of Safe-Control-Gym}

\begin{table}[htbp]
\centering
\label{uncertainty_setting}
\caption{The observation, dynamics and action uncertainty settings of Safe-Control-Gym tasks}
\begin{tabular}{@{}c|c|c|cc@{}}
\toprule
\multirow{2}{*}{Uncertainty Object} & \multirow{2}{*}{Type}        & \multirow{2}{*}{Config}               & \multicolumn{2}{c}{System}                               \\ \cmidrule(l){4-5} 
                                    &                              &                                       & \multicolumn{1}{c|}{CartPole} & QuadRotor                \\ \midrule
\multirow{3}{*}{Observation}        & \multirow{3}{*}{white noise} & \multirow{3}{*}{std:{[}-0.1, 0.1{]}}  & $(x, \dot{x})$                & $(x, \dot{x})$           \\
                                    &                              &                                       & $(\theta, \dot{\theta})$      & $(z, \dot{z})$           \\
                                    &                              &                                       &                               & $(\theta, \dot{\theta})$ \\ 
\multirow{2}{*}{Dynamics}           & \multirow{2}{*}{white noise} & \multirow{2}{*}{std:{[}-0.1, 0.1{]}}  & pole length                   & quadrotor mass           \\
                                    &                              &                                       & pole mass                     & quadrotor inertia        \\ 
\multirow{4}{*}{Action} &
  \multirow{4}{*}{Impulse noise} &
  \multicolumn{1}{l|}{Force: \ {[}-1, 1{]}} &
  \multicolumn{1}{c}{\multirow{4}{*}{horizontal forces}} &
  \multicolumn{1}{c}{\multirow{4}{*}{motors thrusts}} \\
                                    &                              & \multicolumn{1}{l|}{Step offset: \ 20}  & \multicolumn{1}{c}{}          & \multicolumn{1}{c}{}     \\
                                    &                              & \multicolumn{1}{l|}{Duration: \ 80}      & \multicolumn{1}{c}{}          & \multicolumn{1}{c}{}     \\
                                    &                              & \multicolumn{1}{l|}{Decary rate: \ 0.9} & \multicolumn{1}{c}{}          & \multicolumn{1}{c}{}     \\
\bottomrule
\end{tabular}
\end{table}

\subsection{More experiment results}
\label{appendix:results} 

\subsubsection{Comparative Analysis of Fuzzy Operator and Min-Max Operator in Safe Reinforcement Learning}

We first formally define three safety sets: the fundamental safety set $\mathcal{S}_c$ represents permissible state constraints, the safe forward invariant set $\mathcal{S_I}$ (a subset of $\mathcal{S}_c$) guarantees persistent state containment within $\mathcal{S}_c$ under nominal conditions, and the robust safe forward invariant set $\mathcal{S_R}$ (a conservative subset of $\mathcal{S_I}$) maintains state invariance under worst-case disturbances. 

\begin{figure}[htbp]
  \centering
  \includegraphics[width=\textwidth]{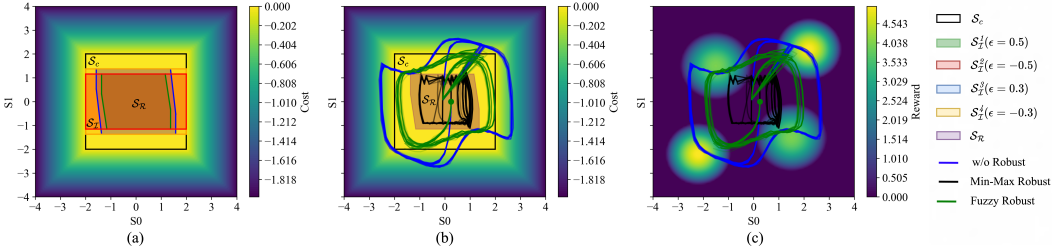}
  \caption{(a) Hierarchical relationship of safety sets, (b) Cost space and (c) Reward space trajectory comparisons, with dashed lines indicating safety boundaries.}
  \label{simple_validation}
\end{figure}

As shown in Fig. \ref{simple_validation}(a), Conventional min-max approaches through robust control barrier functions (RCBFs) strictly confine states within $\mathcal{S_R}$, where $\mathcal{S_R}$ (yellow region) occupies only 23.6\% of $\mathcal{S}_c$ (gray region). This conservative strategy ensures absolute safety at the cost of exploration capability, sacrificing access to 41.7\% of high-reward regions.

Our proposed fuzzy robust method overcomes this limitation through dynamic weighting on different uncertainty levels. The training curves in Fig. \ref{demo_train} demonstrate that in the double-integrator environment, Fuz-RL's value iteration algorithm achieves 2.17× higher final returns compared to the min-max approach under Level-15 configuration. The underlying mechanism enables adaptive safety margin adjustment, permitting safe exploration in $\mathcal{S}_c \setminus \mathcal{S_R}$ regions during 97.4\% of test episodes. Trajectory heatmaps in Fig. \ref{simple_validation}(b)-(c) reveal that while conventional methods (black trajectories) remain strictly confined within $\mathcal{S_R}$, and non-robust approaches (blue trajectories) risk 32.6\% boundary violations, our fuzzy robust method (green trajectories) achieves optimal performance balance with 97\% safety rate through dynamic fuzzy measure.

\begin{figure}[htbp]
  \centering
  \includegraphics[width=0.95\textwidth]{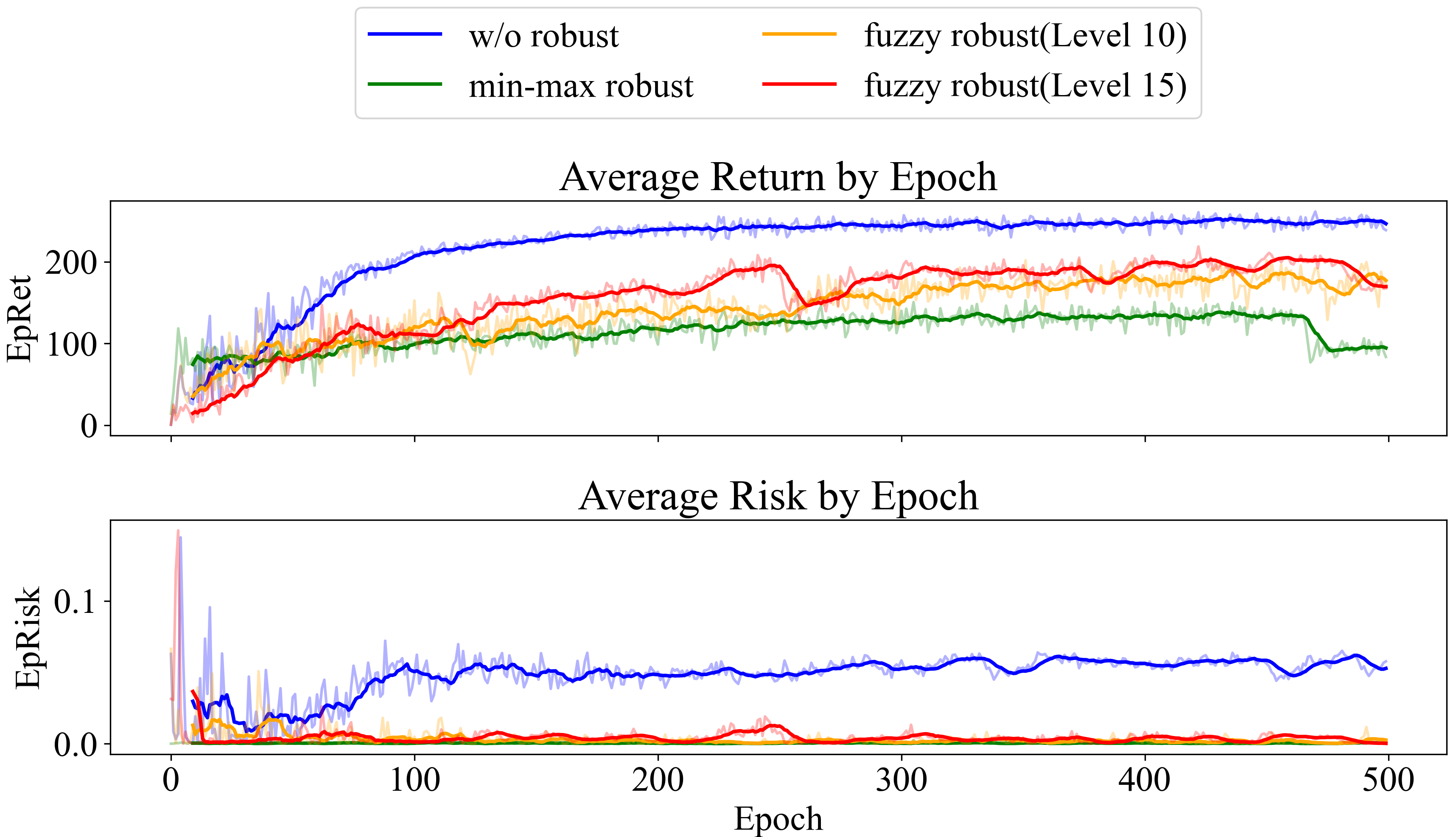}
  \caption{Training curve comparison in double integrator environment, with shaded regions indicating standard deviation across 5 random seeds.}
  \label{demo_train}
\end{figure}

\subsubsection{Comparison between Safe RL and Fuz-RL}
Similar to the Quadrotor-Track task, we set different levels of perturbations in the observation, dynamics, and action to evaluate the performance of the CartPole-Stab, CartPole-Track, and Quadrotor-Stab tasks under the three benchmark safe RL algorithms and the corresponding Fuz-RL, as shown in Figures \ref{cartpole_track}, \ref{quadrotor_stab}, and \ref{quadrotor_track}. Each point in the figures represents the average metrics from 10 episodes run for each of 10 different seeds.

\subsubsection{Comparison between Fuz-RL and RAMU}
To compare with the current SOTA algorithms in robust safe RL, this section showcases the performance comparison between RAMU and Fuz-RL under uncertainties in observation, action, and dynamics, as depicted in Figures \ref{RAMU_obs}, Figures \ref{RAMU_act}, and Figures \ref{RAMU_dyn}, respectively.

\subsubsection{Validation on Power System Frequency Control Task}
\label{app::powersystem}
The IEEE 39-bus system, a standard power grid benchmark with 10 generators and 46 transmission lines, was used to validate Fuz-RL's performance in frequency control tasks. The system state captures frequency deviations ($\Delta f$), generator rotor angles ($\delta$), mechanical power outputs ($P_m$), and tie-line power flows ($P_{\text{tie}}$). Control actions involve real-time adjustments of generator active power setpoints ($P_{\text{ref}}$) and discrete load shedding commands (0-100\% reduction). The primary objectives are to maintain frequency within $[59.8\,\text{Hz}, 60.2\,\text{Hz}]$ under stochastic load/renewable fluctuations while minimizing control costs ($\sum \|P_{\text{ref}} - P_{\text{nominal}}\|_2$) and avoiding safety-critical violations such as line overloads (>120\% capacity).

Robustness tests were conducted under three uncertainty scenarios: 
\begin{itemize}
\item \textit{Observation noise} ($\sigma=0.1\,\text{Hz}$ Gaussian noise in frequency measurements),
\item \textit{Action noise} (100ms delay + 5\% bias in control signals),
\item \textit{Dynamics noise} ($\pm10\%$ parameter drift in generator inertia/damping).
\end{itemize}

\begin{table}[ht]
\centering
\caption{Performance on IEEE 39-Bus Frequency Control (AvgRet / AvgRisk)}
\label{tab:power_system}
\begin{tabular}{@{}llccc@{}}
\toprule
\textbf{Case} & \textbf{Method} & \textbf{Observation Noise} & \textbf{Action Noise} & \textbf{Dynamic Noise} \\ 
\midrule
IEEE-39 Bus & PPOL & -5456.30 / 0.17 & -6357.81 / 0.16 & -7471.96 / 0.52 \\
IEEE-39 Bus & \textbf{Fuz-PPOL} & \textbf{-4822.03 / 0.14} & \textbf{-5789.19 / 0.13} & \textbf{-7363.20 / 0.47} \\
\bottomrule
\end{tabular}
\end{table}

Fuz-PPOL demonstrates consistent improvements over PPOL. Under observation noise, Fuz-PPOL get 11.6\% higher returns and 17.6\% lower risk. For action noise, the AvgRet metic is improved by 8.9\% with 18.8\% risk reduction. Under dynamics perturbations, Fuz-PPOL narrows performance degradation while reducing safety violations by 9.6\%. 

\begin{figure}[t]
  \centering
  \includegraphics[width=1\columnwidth]{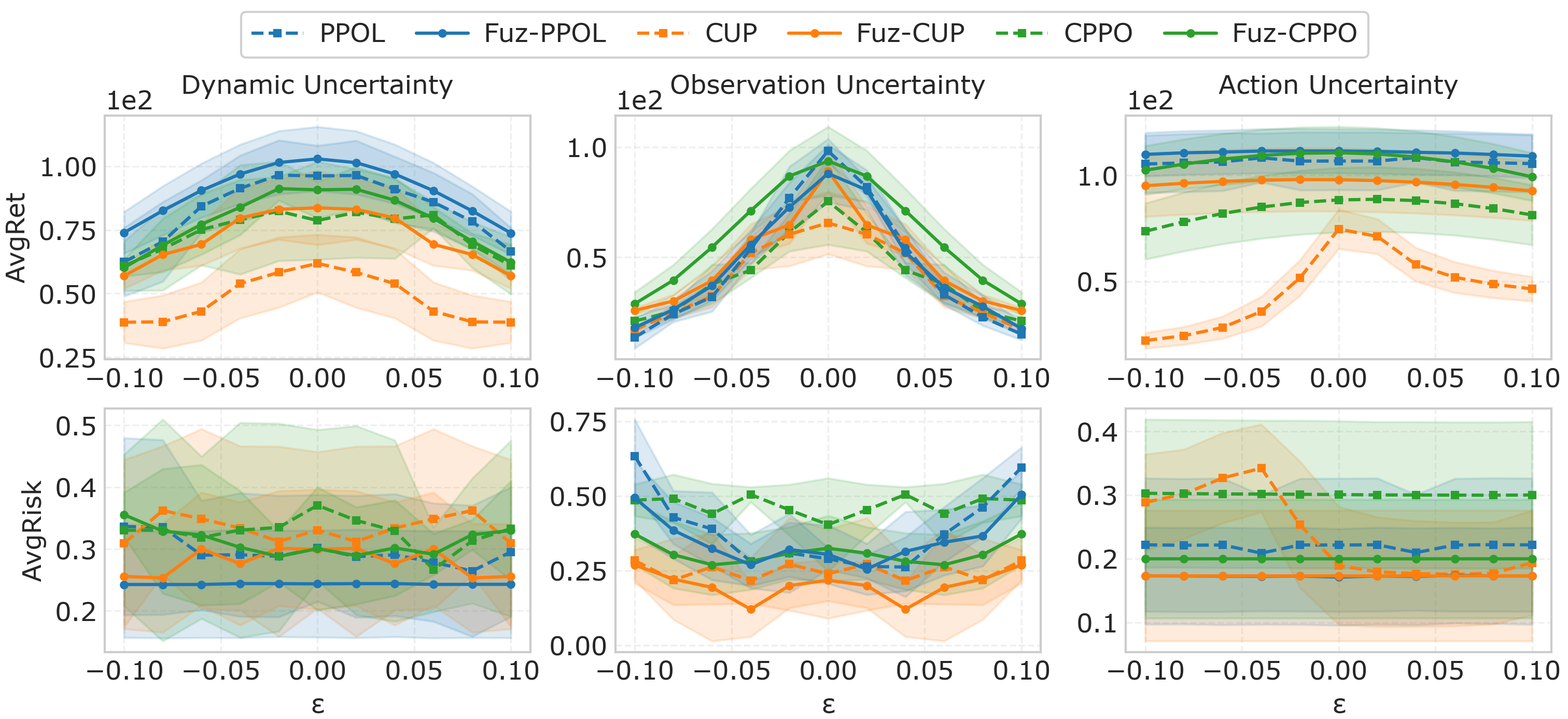} 
  \caption{Average episodic rewards and average episodic risk of three safe RL and Fuz-RL under various uncertainty settings in Cartpole Stab task.}
  \label{cartpole_track}
\end{figure}

\begin{figure}[t]
  \centering
  \includegraphics[width=1\columnwidth]{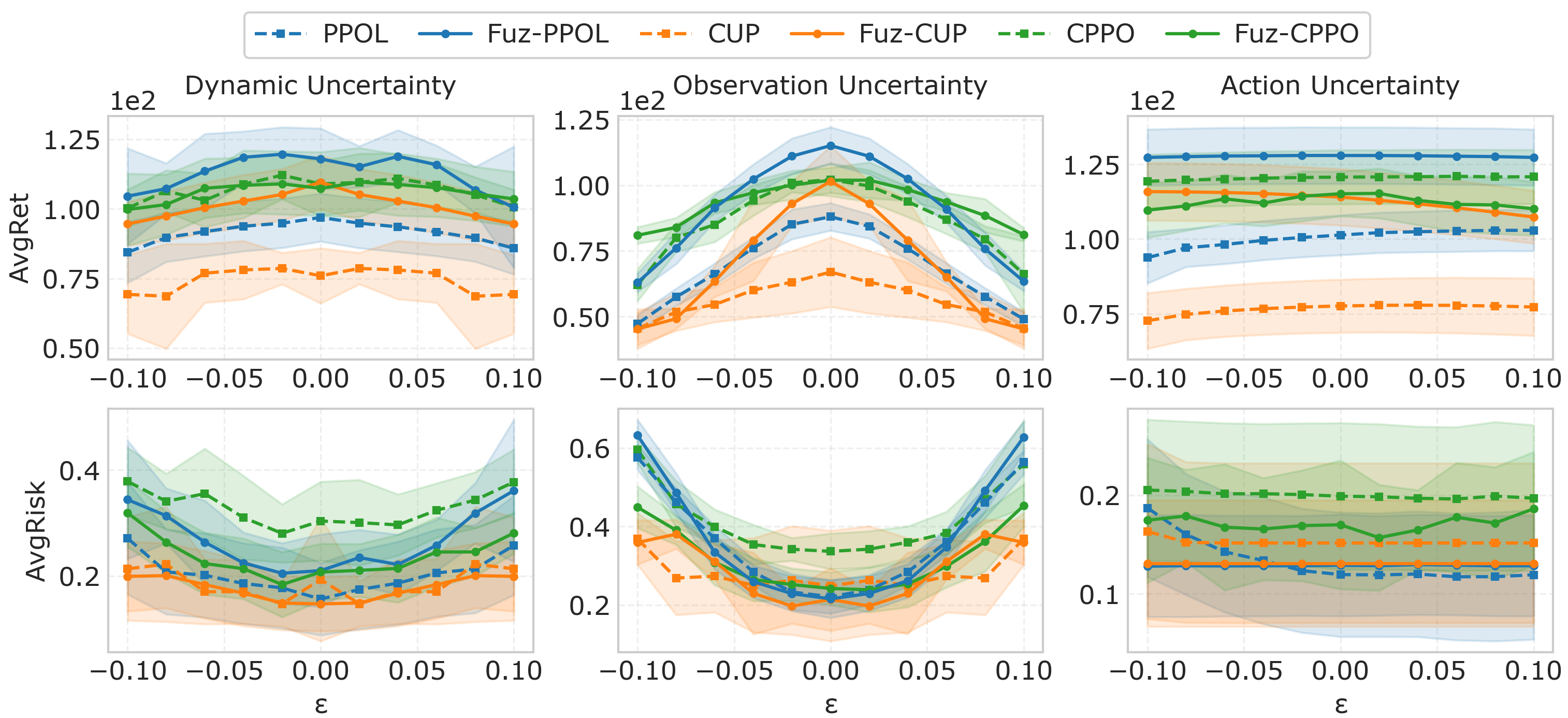} 
  \caption{Average episodic rewards and average episodic risk of three safe RL and Fuz-RL under various uncertainty settings in Cartpole Track task.}
  \label{quadrotor_stab}
\end{figure}

\begin{figure}[t]
  \centering
  \includegraphics[width=1\columnwidth]{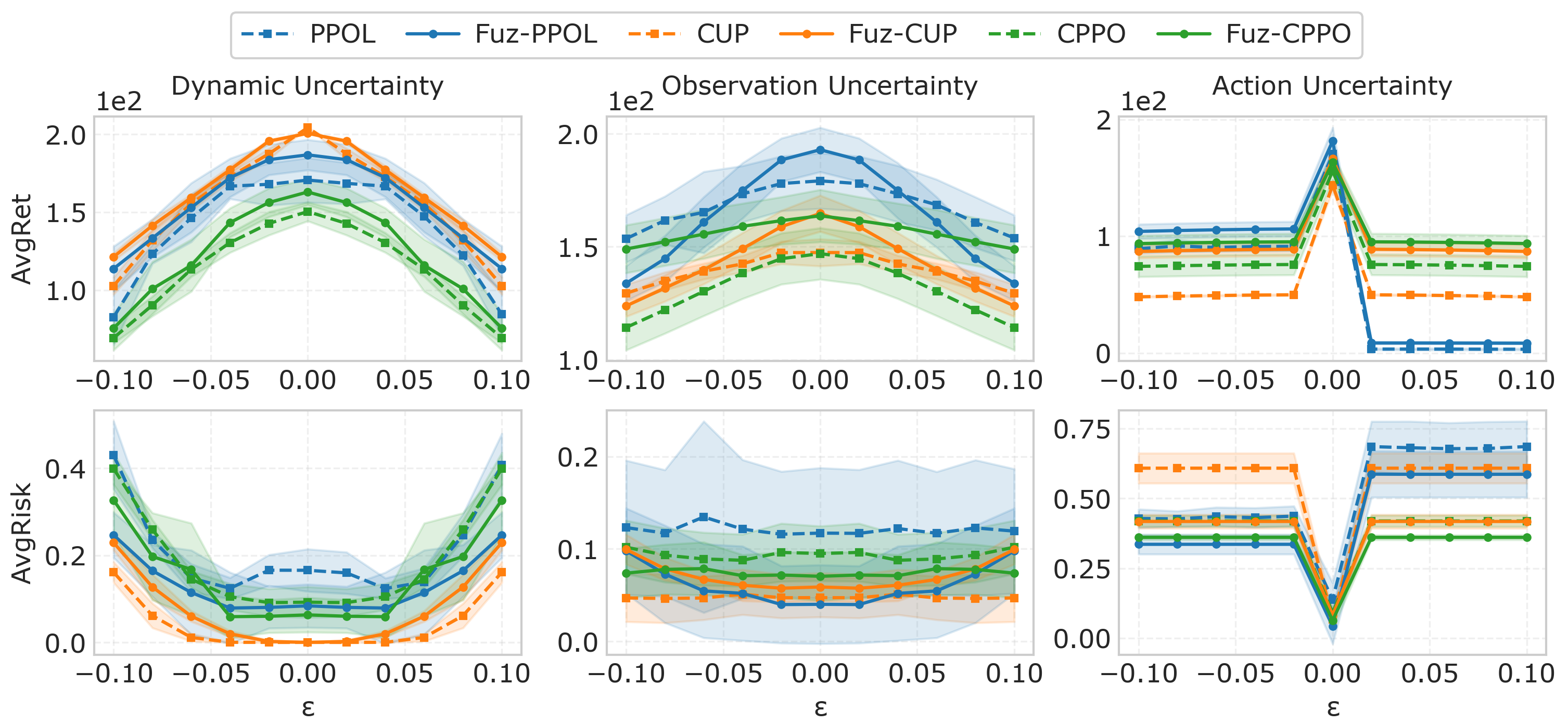} 
  \caption{Average episodic rewards and average episodic risk of three safe RL and Fuz-RL under various uncertainty settings in Quadrotor Stab task.}
  \label{quadrotor_track}
\end{figure}

\begin{figure}[t]
  \centering
  \includegraphics[width=\columnwidth]{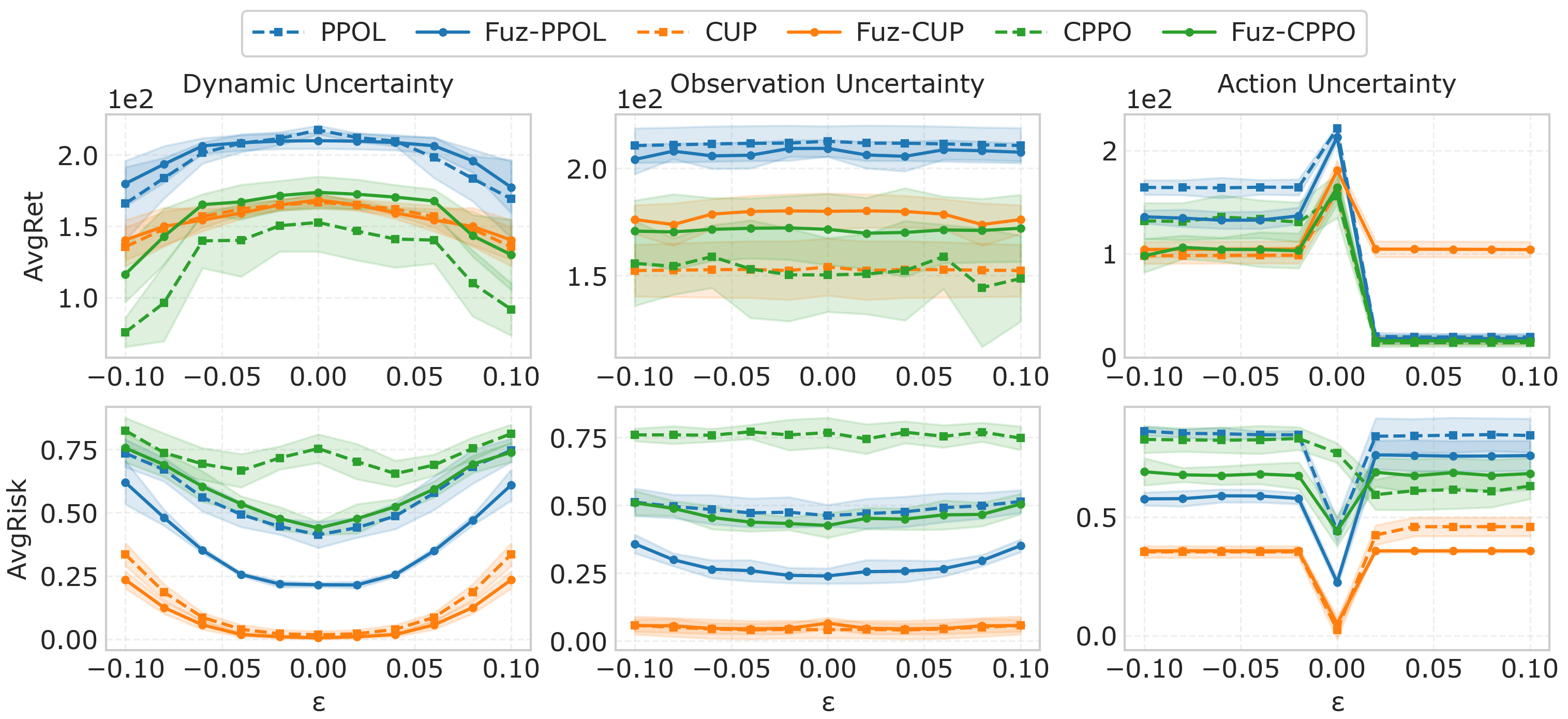} 
  \caption{Average episodic rewards and average episodic risk of three safe RL and Fuz-RL under various uncertainty settings in Quadrotor Track task.}
  \label{fig:quadrotor_stab}
\end{figure}

\begin{figure}[t]
  \centering
  \includegraphics[width=\columnwidth]{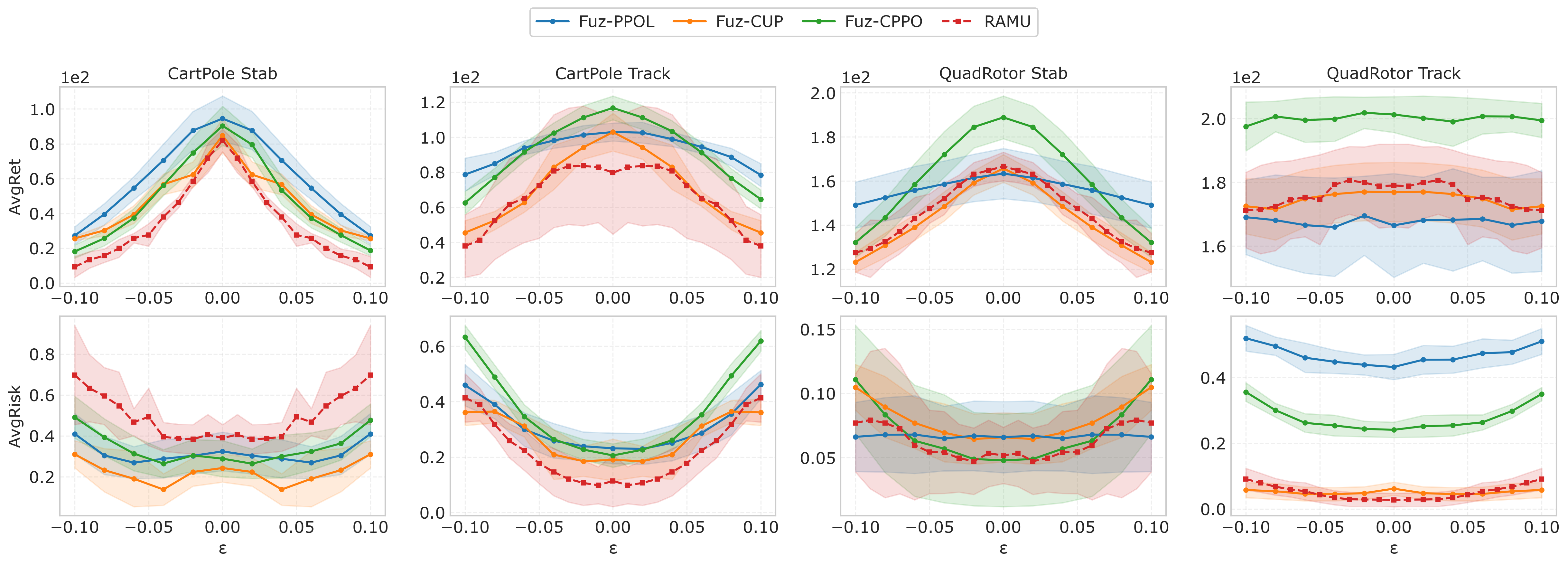} 
  \caption{Average episodic reward (top) and average episodic risk (bottom) of Fuz-RL and RAMU in different scales' observation uncertainty settings. The horizontal axis represents the uncertainty level.}
  \label{RAMU_obs}
\end{figure}

\begin{figure}[t]
  \centering
  \includegraphics[width=\columnwidth]{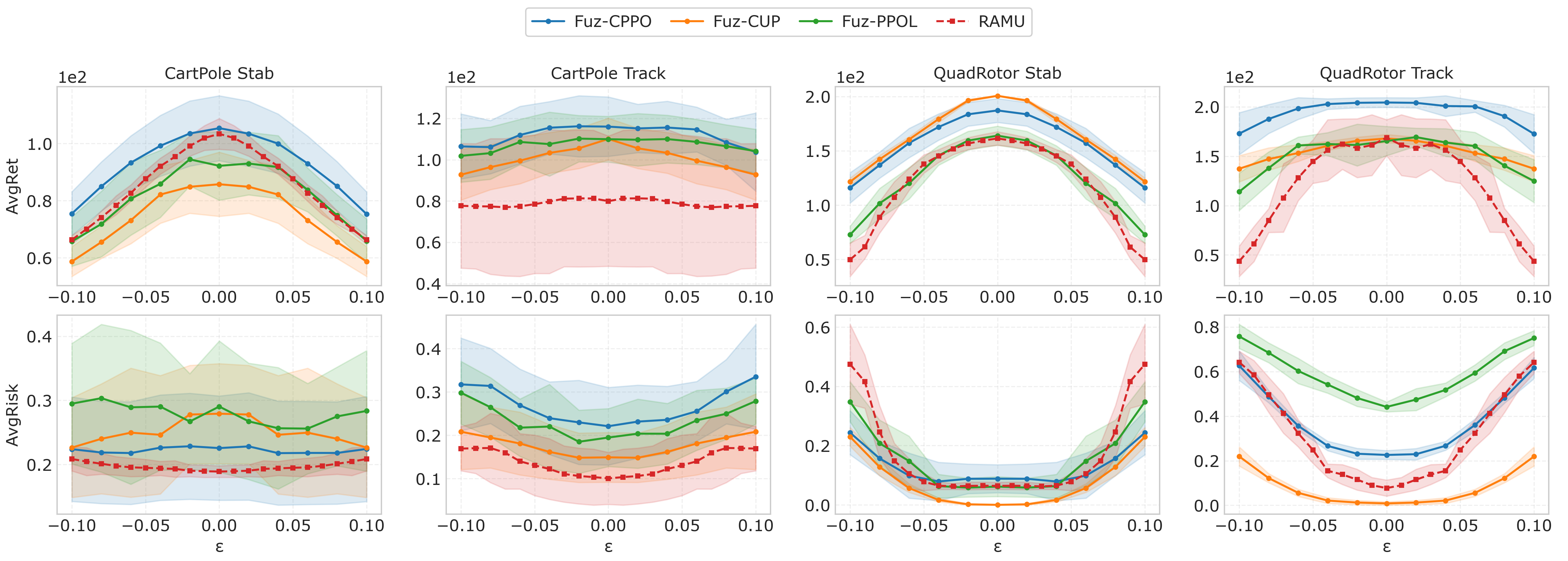} 
  \caption{Average episodic reward (top) and average episodic risk (bottom) of Fuz-RL and RAMU in different scales' action uncertainty settings. The horizontal axis represents the uncertainty level.}
  \label{RAMU_act}
\end{figure}

\begin{figure}[t]
  \centering
  \includegraphics[width=\columnwidth]{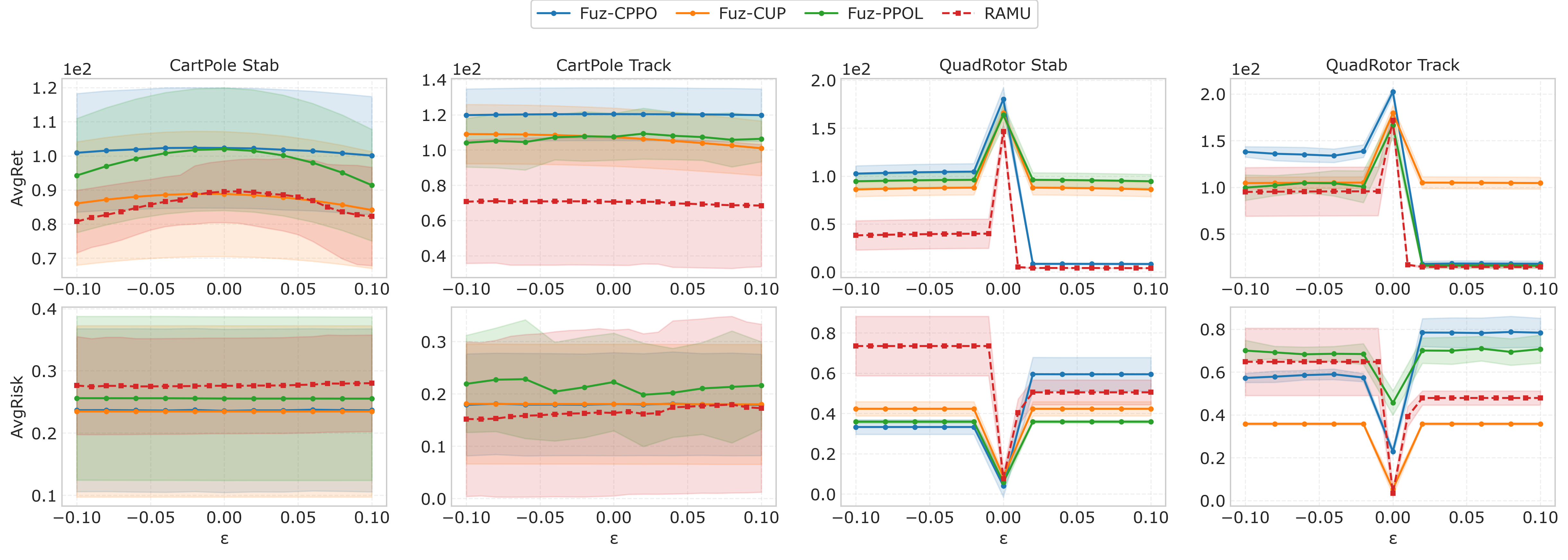} 
  \caption{Average episodic reward (top) and average episodic risk (bottom) of Fuz-RL and RAMU in different scales' dynamics uncertainty settings. The horizontal axis represents the uncertainty level.}
  \label{RAMU_dyn}
\end{figure}

\begin{figure}[t]
  \centering
  \includegraphics[width=\columnwidth]{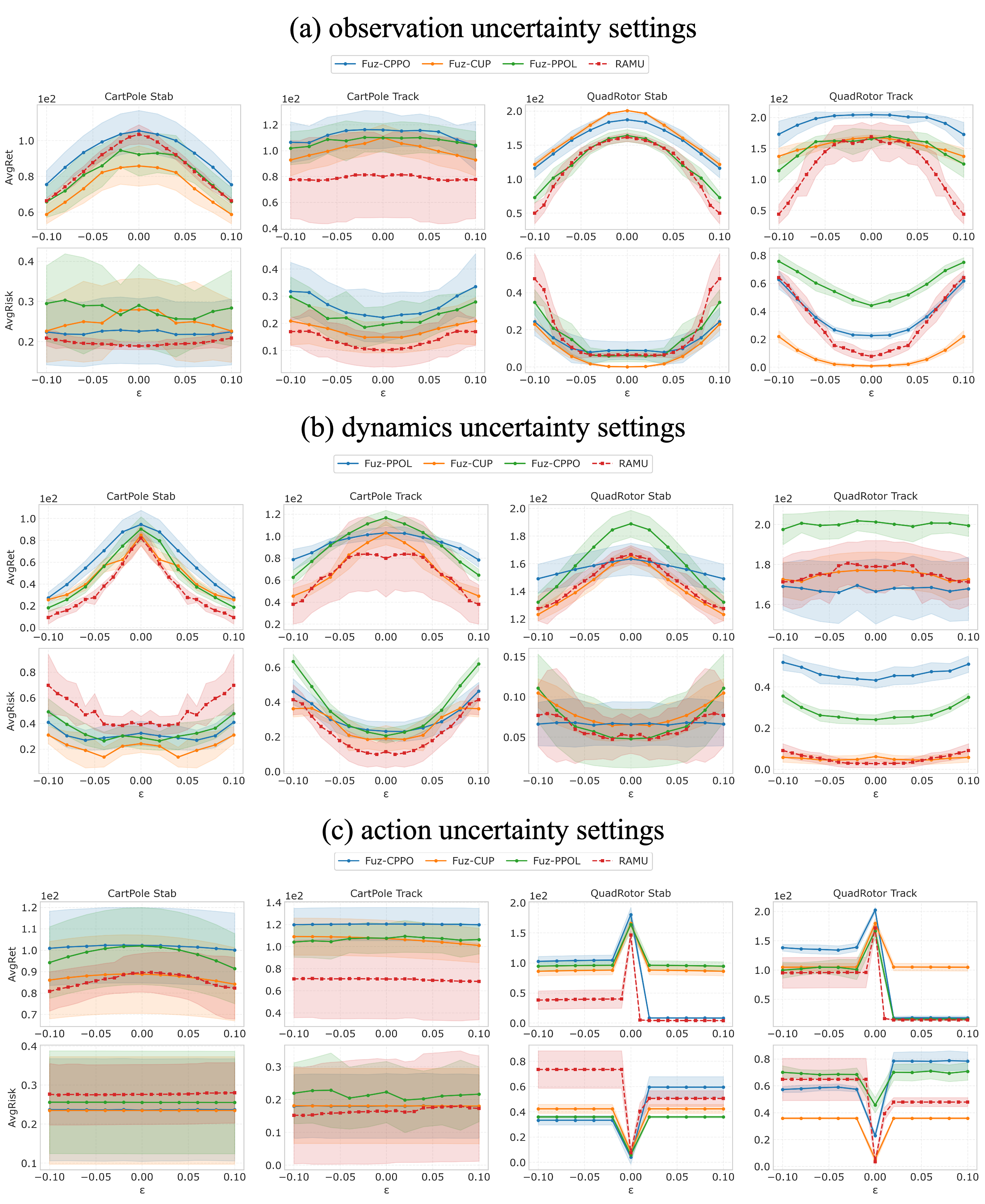} 
  \caption{Average episodic reward (top) and average episodic risk (bottom) of Fuz-RL and RAMU in different scales' observation, dynamics, and action uncertainty settings. The horizontal axis represents the uncertainty level.}
  \label{fig::all_comp}
\end{figure}

\clearpage


\newpage
\section*{NeurIPS Paper Checklist}

\begin{enumerate}

\item {\bf Claims}
    \item[] Question: Do the main claims made in the abstract and introduction accurately reflect the paper's contributions and scope?
    \item[] Answer: \answerYes{}
    \item[] Justification: We clearly indicate the main contributions of our Fuz-RL  scope in the abstract and introduction.

\item {\bf Limitations}
    \item[] Question: Does the paper discuss the limitations of the work performed by the authors?
    \item[] Answer: \answerYes{} 
    \item[] Justification: In Section \ref{conclu}, we describe the limitations of the proposed Fuz-RL.

\item {\bf Theory assumptions and proofs}
    \item[] Question: For each theoretical result, does the paper provide the full set of assumptions and a complete (and correct) proof?
    \item[] Answer: \answerYes{} 
    \item[] Justification: The assumptions of the theoretical result are clearly stated in the Theorem. \ref{thm:convergence}, Theorem. \ref{thm:gamma_contraction} and Theorm. \ref{theorem:equivalent}. Refer to Appendix~\ref{app:proofs} for the complete proof.

    \item {\bf Experimental result reproducibility}
    \item[] Question: Does the paper fully disclose all the information needed to reproduce the main experimental results of the paper to the extent that it affects the main claims and/or conclusions of the paper (regardless of whether the code and data are provided or not)?
    \item[] Answer: \answerYes{} 
    \item[] Justification: We provide a thorough explanation in Section~\ref{appendix:Environment_description} of all our experiments' settings, and the more general setting and detailed results can be found in Appendix~\ref{appendix:Hyper_parameters}.

\item {\bf Open access to data and code}
    \item[] Question: Does the paper provide open access to the data and code, with sufficient instructions to faithfully reproduce the main experimental results, as described in supplemental material?
    \item[] Answer: \answerYes{} 
    \item[] Justification: The whole code is provided in the supplemental material with sufficient instructions in the ``README.md'' file.

\item {\bf Experimental setting/details}
    \item[] Question: Does the paper specify all the training and test details (e.g., data splits, hyperparameters, how they were chosen, type of optimizer, etc.) necessary to understand the results?
    \item[] Answer: \answerYes{} 
    \item[] Justification: Appendix~\ref{appendix:Hyper_parameters} clearly states the implementation details.

\item {\bf Experiment statistical significance}
    \item[] Question: Does the paper report error bars suitably and correctly defined or other appropriate information about the statistical significance of the experiments?
    \item[] Answer: \answerYes{} 
    \item[] Justification: We evaluate our experimental results across 10 episodes with 10 different random seeds, reporting standard deviations as error bars in the figures and as variance metrics in the tables.

\item {\bf Experiments compute resources}
    \item[] Question: For each experiment, does the paper provide sufficient information on the computer resources (type of compute workers, memory, time of execution) needed to reproduce the experiments?
    \item[] Answer: \answerYes{} 
    \item[] Justification: We indicate the information of compute resources in appendix~\ref{appendix:results}.
    
\item {\bf Code of ethics}
    \item[] Question: Does the research conducted in the paper conform, in every respect, with the NeurIPS Code of Ethics \url{https://neurips.cc/public/EthicsGuidelines}?
    \item[] Answer: \answerYes{} 
    \item[] Justification:  We strictly adhere to the code of ethics.

\item {\bf Broader impacts}
    \item[] Question: Does the paper discuss both potential positive societal impacts and negative societal impacts of the work performed?
    \item[] Answer: \answerYes{} 
    \item[] Justification: The potential impacts are stated in the Conclusion, see Section ~\ref{conclu}.
    
\item {\bf Safeguards}
    \item[] Question: Does the paper describe safeguards that have been put in place for responsible release of data or models that have a high risk for misuse (e.g., pretrained language models, image generators, or scraped datasets)?
    \item[] Answer: \answerNA{} 
    \item[] Justification: This paper poses no such risks.

\item {\bf Licenses for existing assets}
    \item[] Question: Are the creators or original owners of assets (e.g., code, data, models), used in the paper, properly credited and are the license and terms of use explicitly mentioned and properly respected?
    \item[] Answer: \answerYes{} 
    \item[] Justification: All codes of Fuz-RL are implemented based on the SpinningUp \cite{achiam2018spinning} which is explicitly mentioned and properly respected in Section \ref{main:exp}.

\item {\bf New assets}
    \item[] Question: Are new assets introduced in the paper well documented and is the documentation provided alongside the assets?
    \item[] Answer: \answerYes{} 
    \item[] Justification: New assets are attached in the supplemental material.

\item {\bf Crowdsourcing and research with human subjects}
    \item[] Question: For crowdsourcing experiments and research with human subjects, does the paper include the full text of instructions given to participants and screenshots, if applicable, as well as details about compensation (if any)? 
    \item[] Answer: \answerNA{} 
    \item[] Justification: This paper does not include crowdsourcing nor research with human subjects.

\item {\bf Institutional review board (IRB) approvals or equivalent for research with human subjects}
    \item[] Question: Does the paper describe potential risks incurred by study participants, whether such risks were disclosed to the subjects, and whether Institutional Review Board (IRB) approvals (or an equivalent approval/review based on the requirements of your country or institution) were obtained?
    \item[] Answer: \answerNA{}  
    \item[] Justification: This paper does not involve any crowdsourcing nor research with human subjects.

\item {\bf Declaration of LLM usage}
    \item[] Question: Does the paper describe the usage of LLMs if it is an important, original, or non-standard component of the core methods in this research? Note that if the LLM is used only for writing, editing, or formatting purposes and does not impact the core methodology, scientific rigorousness, or originality of the research, declaration is not required.
    \item[] Answer: \answerNA{} 
    \item[] Justification: The core method development in Fuz-RL does not involve LLMs as any important, original, or non-standard components.

\end{enumerate}

\end{document}